\documentclass[11pt]{article} 
\usepackage{url}
\usepackage{smile}
\usepackage{graphicx} 
\usepackage{algorithm}
\usepackage{algorithmic}
\usepackage{todonotes}
\usepackage{epstopdf}
\usepackage{wrapfig}
\usepackage[colorlinks, linkcolor=blue, anchorcolor=blue, citecolor=blue]{hyperref}
\usepackage[margin=1in]{geometry}
\usepackage[normalem]{ulem}
\usepackage[export]{adjustbox}
\usepackage{mathtools, cuted}
\usepackage{natbib}
\usepackage{enumitem}
\usepackage{microtype}
\usepackage{graphicx}
\usepackage{booktabs} 
\usepackage{smile}
\usepackage{wrapfig}
\usepackage{bbm}
\usepackage{subfigure}
\usepackage{listings}
\usepackage[T1]{fontenc}

\usepackage{kpfonts}
\DeclareMathAlphabet{\mathsf}{OT1}{cmss}{m}{n}
\SetMathAlphabet{\mathsf}{bold}{OT1}{cmss}{bx}{n}

\newcommand{\myCOMMENT}[2][.25\linewidth]{%
  \leavevmode\hfill\makebox[#1][l]{\#~#2}}

\providecommand{\norm}[1]{\|#1\|}

\begin{document}

\title{\huge \bf Differentiable Top-$k$ Operator with Optimal Transport}

\author{Yujia Xie, Hanjun Dai, Minshuo Chen, Bo Dai, Tuo Zhao, Hongyuan Zha, \\Wei Wei, Tomas Pfister \thanks{Work done when Yujia was an intern in Google. Yujia Xie, Minshuo Chen, Tuo Zhao are affiliated with Georgia Institute of Technology. Hanjun Dai, Bo Dai, Wei Wei, Tomas Pfister are affiliated with Google Inc. Yujia Xie and Wei Wei are the corresponding authors. Emails: {\tt Xie.Yujia000@gmail.com, wewei@google.com}.}}
\date{}

\maketitle


\begin{abstract}

The top-$k$ operation, i.e., finding the $k$ largest or smallest elements from a collection of scores, is an important model component, which is widely used in information retrieval, machine learning, and data mining.  
However, if the top-$k$ operation is implemented in an algorithmic way, e.g., using bubble algorithm, the resulting model cannot be trained in an end-to-end way using prevalent gradient descent algorithms. This is because these implementations typically involve swapping indices, whose gradient cannot be computed. 
Moreover, the corresponding mapping from the input scores to the indicator vector of whether this element belongs to the top-$k$ set is essentially discontinuous.
To address the issue, we propose a smoothed approximation, namely the SOFT (Scalable Optimal transport-based diFferenTiable) top-$k$ operator.
Specifically, our SOFT top-$k$ operator approximates the output of the top-$k$ operation as the solution of an Entropic Optimal Transport (EOT) problem. 
The gradient of the SOFT operator can then be efficiently approximated based on the optimality conditions of EOT problem.
We apply the proposed operator to the $k$-nearest neighbors and beam search algorithms, and demonstrate improved performance.
\end{abstract}

\section{Introduction}
\label{sec:intro}

The top-$k$ operation, i.e., finding the $k$ largest or smallest elements from a set, is widely used for predictive modeling in information retrieval, machine learning, and data mining.
For example, in image retrieval \citep{babenko2014neural, radenovic2016cnn, gordo2016deep}, one needs to query the $k$ nearest neighbors of an input image under certain metrics;
in the beam search \citep{reddy1977speech, wiseman2016sequence} algorithm for neural machine translation, one needs to find the $k$ sequences of largest likelihoods in each decoding step.

Although the ubiquity of top-$k$ operation continues to grow, the operation itself is difficult to be integrated into the training procedure of a predictive model. For example, we consider a neural network-based $k$-nearest neighbor classifier. Given an input, we use the neural network to extract features from the input. Next, the extracted features are fed into the top-$k$ operation for identifying the $k$ nearest neighbors under some distance metric. We then obtain a prediction based on the $k$ nearest neighbors of the input. In order to train such a model, we choose a proper loss function, and minimize the average loss across training samples using (stochastic) first-order methods. This naturally requires the loss function being differentiable with respect to the input at each update step. Nonetheless, the top-$k$ operation does not exhibit an explicit mathematical formulation: most implementations of the top-$k$ operation, e.g., bubble algorithm and Q{\scriptsize UICKSELECT} \citep{hoare1962quicksort}, involve operations on indices such as indices swapping. Consequently, the training objective is difficult to  formulate explicitly.

%
Alternative perspective --- taking the top-$k$ operation as an operator --- still cannot resolve the differentibility issue. Specifically, the top-$k$ operator\footnote{Throughout the rest of the paper, we refer to the top-$k$ operator as the top-$k$ operation.} maps a set of inputs $x_1, \dots, x_n$ to an index vector $\{0, 1\}^n$. Whereas the Jacobian matrix of such a mapping is not well defined. As a simple example, 
 \begin{figure}
 \vspace{-11pt}
\centering 
     \includegraphics[width=0.26\textwidth]{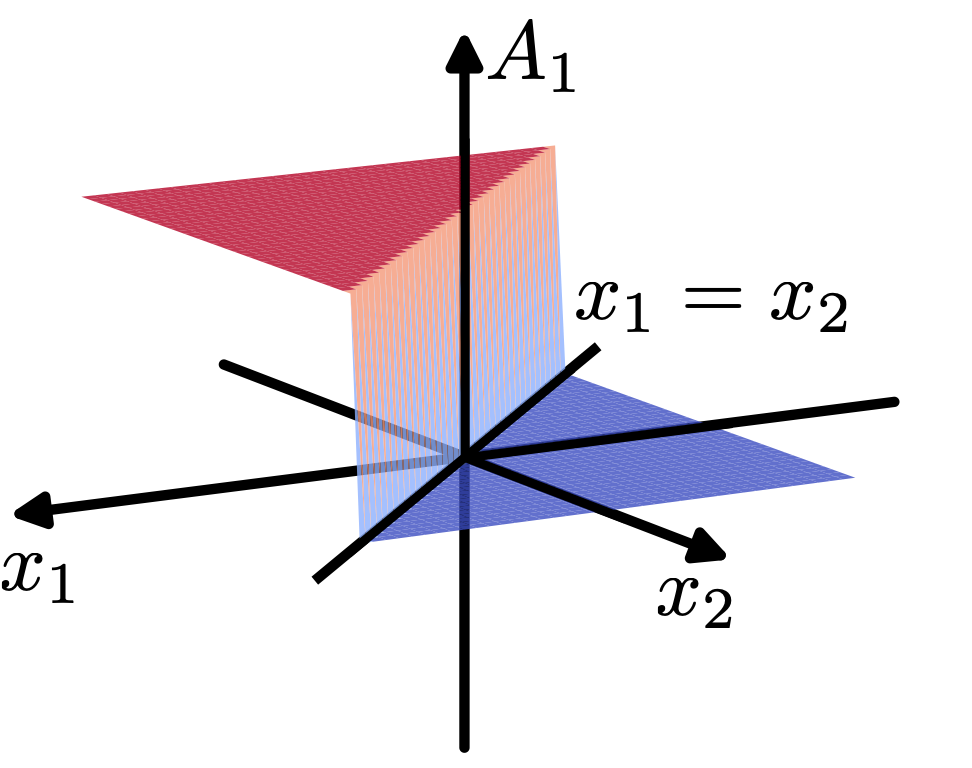}
      \includegraphics[width=0.26\textwidth]{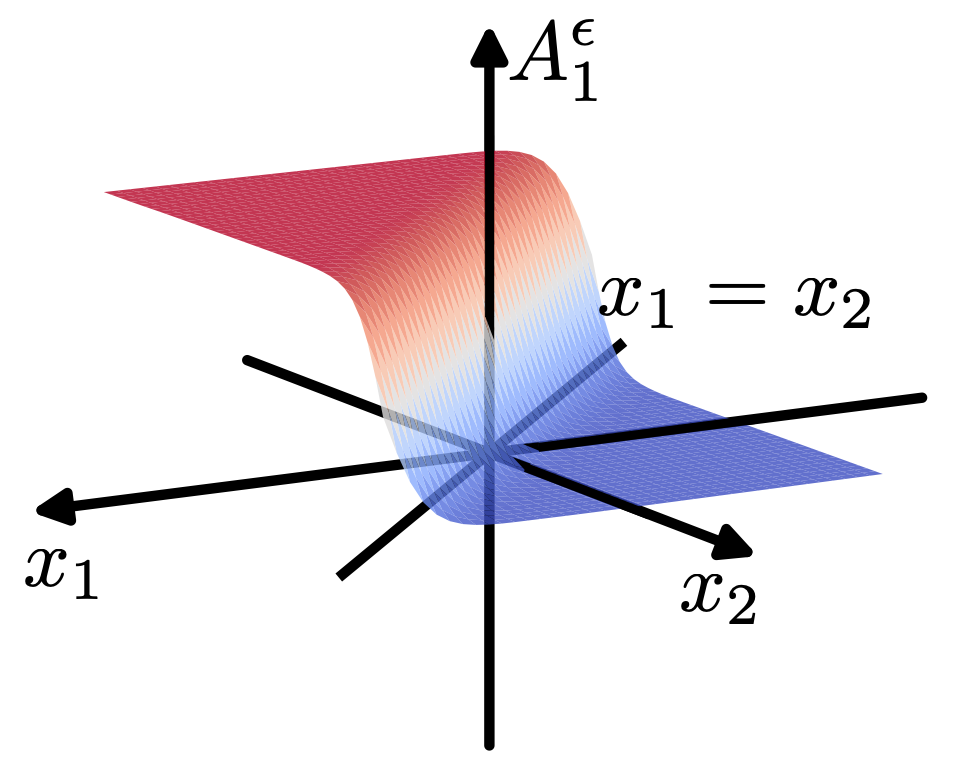}
     \vspace{-13pt}
   \caption{\label{fig:illu_step_function} Indicator vector with respect to input scores. Left: original top-$k$ operator; right: SOFT top-$k$ operator.}
   \vspace{-0.15in}
 \end{figure}
consider two scalars $x_1$, $x_2$. The top-$1$ operation as in Figure \ref{fig:illu_step_function} returns a vector $[A_1, A_2]^\top$, with each entry denoting whether the scalar is the larger one ($1$ for true, $0$ for false). Denote $A_1 = f(x_1, x_2)$. For a fixed $x_2$, $A_1$ jumps from $0$ to $1$ at $x_1 = x_2$. It is clear that $f$ is not differentiable at $x_1 = x_2$, and the derivative is identically zero otherwise.

Due to the aforementioned difficulty, existing works resort to two-stage training for models with the top-$k$ operation. We consider the neural network-based $k$-nearest neighbor classifier again. As proposed in \citet{papernot2018deep}, one first trains the neural network using some surrogate loss on the extracted features, e.g., using softmax activation in the output layer and the cross-entropy loss. Next, one uses the $k$-nearest neighbor for prediction based on the features extracted by the well-trained neural network. This training procedure, although circumventing the top-$k$ operation, makes the training and prediction misaligned; and the actual performance suffers.


In this work, we propose the SOFT (Scalable Optimal transport-based diFferenTiable) top-$k$ operation as a differentiable approximation of the standard top-$k$ operation in Section.~\ref{sec:method}. Specifically, motivated by the implicit differentiation \citep{duchi2008efficient, griewank2008evaluating, amos2017optnet, luise2018differential} techniques, we first parameterize the top-$k$ operation in terms of the optimal solution of an Optimal Transport (OT) problem. Such a re-parameterization is still not differentiable with respect to the input. To rule out the discontinuity, we impose entropy regularization to the optimal transport problem, and show that the optimal solution to the Entropic OT (EOT) problem yields a differentiable approximation to the top-$k$ operation. Moreover, we prove that under mild assumptions, the approximation error can be properly controlled.




We then develop an efficient implementation of the SOFT top-$k$ operation in Section.~\ref{sec:method2}. Specifically, we solve the EOT problem via the Sinkhorn algorithm \citep{cuturi2013sinkhorn}. Given the optimal solution, we can explicitly formulate the gradient of SOFT top-$k$ operation using the KKT (Karush-Kuhn-Tucker) condition. As a result, the gradient at each update step can be efficiently computed with complexity $\cO (n)$, where $n$ is the number of elements in the input set to the top-$k$ operation.




Our proposed SOFT top-$k$ operation allows end-to-end training, and we apply SOFT top-$k$ operation to $k$NN for classification in Section~\ref{sec:app},  beam search in Section~\ref{sec:app2}  and learning sparse attention for neural machine translation in Section~\ref{sec:app3}.
The experimental results demonstrate significant performance gain over competing methods, as an end-to-end training procedure resolves the misalignment between training and prediction. 

\noindent{\bf Notations.} We denote $\norm{\cdot}_2$ as the $\ell_2$ norm of vectors, $\norm{\cdot}_{\textrm{F}}$ as the Frobenius norm of matrices. Given two matrices $B, D\in \mathbb{R}^{n\times m}$, we denote $\langle B, D \rangle$ as the inner product, i.e., $\langle B, D \rangle = \sum_{i=1,j=1}^{n,m} B_{ij}D_{ij}$. We denote $B\odot D$ as the element-wise multiplication of $B$ and $D$. We denote $\mathbbm{1}(\cdot)$ as the indicator function, i.e., the output of $\mathbbm{1}(\cdot)$ is $1$ if the input condition is satisfied, and is $0$ otherwise. For matrix $B\in \mathbb{R}^{n\times m}$, we denote $B_{i,:}$ as the $i$-th row of the matrix. The softmax function for matrix $B$ is defined as ${\rm softmax}_i(B_{ij}) = e^{B_{ij}} / \sum_{\ell=1}^n e^{B_{lj}}$. For a vector $b\in\mathbb{R}^n$, we denote ${\rm diag}(b)$ as the matrix where the $i$-th diagonal entries is $b_i$.


\section{SOFT Top-$k$ Operator}
\label{sec:method}


In this section we derive the proposed SOFT (Scalable Optimal transport-based diFferenTialble) top-$k$ operator.
\vspace{-0.1in}
\subsection{Problem Statement}
\vspace{-0.05in}

Given a set of scalars $\cX = \{x_i\}_{i=1}^n$, the standard top-$k$ operator returns a vector $A = [A_1, \dots, A_n]^\top$, such that 
\begin{align*}
    A_i=\begin{cases}1,\quad \text{if $x_i$ is a top-$k$ element in $\cX$,} \\
    0, \quad\text{otherwise}. \end{cases}
\end{align*}
In this work, our goal is to design a smooth relaxation of the standard top-$k$ operator, whose Jacobian matrix exists and is smooth.
 Without loss of generality, we refer to top-$k$ elements as the \textit{smallest} $k$ elements.

\subsection{Parameterizing Top-$k$ Operator as OT Problem}

We first show that the standard top-$k$ operator can be parameterized in terms of the solution of an Optimal Transport (OT) problem \citep{monge1781memoire, kantorovich1960mathematical}. We briefly introduce OT problems for self-containedness. An OT problem finds a transport plan between two distributions, while the expected cost of the transportation is minimized. We consider two discrete distributions defined on supports $\mathcal{A}=\{a_i\}_{i=1}^n$ and $\mathcal{B}=\{b_j\}_{j=1}^m$, respectively. Denote $\PP(\{a_i\}) = \mu_i$ and $\PP(\{b_j\}) = \nu_j$, and let $\mu = [\mu_1, \dots, \mu_n]^\top$ and $\nu = [\nu_1, \dots, \nu_m]^\top$. We further denote $C \in \RR^{n \times m}$ as the cost matrix with $C_{ij}$ being the cost of transporting mass from $a_i$ to $b_j$. An OT problem can be formulated as
\begin{align}
\label{eq:kanto}
& \Gamma^* = \argmin_{\Gamma \geq 0}  \langle C, \Gamma \rangle, \quad {\rm s.t.,}~~ \Gamma \bm{1}_m = \mu, ~\Gamma^\top \bm{1}_n = \nu,
\end{align}
where $\bm{1}$ denotes a vector of ones. The optimal solution $\Gamma^*$ is referred to as the \textit{optimal transport plan}.

In order to parameterize the top-$k$ operator using the optimal transport plan $\Gamma^*$, we set the support $\cA=\cX$ and $\cB = \{0, 1\}$ in \eqref{eq:kanto}, with $\mu, \nu$ defined as 
\begin{align*}
\mu=\bm{1}_n/n, \quad \nu=[k/n, (n-k)/n]^\top.
\end{align*}
We take the cost to be the squared Euclidean distance, i.e., $C_{i1} = x_i^2$ and $C_{i2} = (x_i - 1)^2$ for $i = 1, \dots, n$. We then establish the relationship between the output $A$ of the top-$k$ operator and $\Gamma^*$.


\begin{proposition}
Consider the setup in the previous paragraph. Without loss of generality, we assume $\cX$ has no duplicates. Then the optimal transport plan $\Gamma^*$ of \eqref{eq:kanto} is
\begin{align}
    &\Gamma^*_{\sigma_i, 1}= \begin{cases}
    1/n, & {\rm ~if~} i\leq k, \\
    0, & {\rm ~if~} k+1 \leq i\leq n.
    \end{cases} \label{eq:gamma1}\\
    &\Gamma^*_{\sigma_i, 2}= \begin{cases}
    0, & {\rm ~if~} i\leq k, \\
    1/n, & {\rm ~if~} k+1 \leq i\leq n,
    \end{cases} \label{eq:gamma2}
\end{align}
with $\sigma$ being the sorting permutation, i.e., $x_{\sigma_1}<x_{\sigma_2}<\cdots<x_{\sigma_n}$. Moreover, we have
\begin{align}
    A=n\Gamma^*\cdot [1,0]^\top.
\end{align}
\end{proposition}
\begin{proof}
We expand the objective function of \eqref{eq:kanto} as
\begin{align*}
    \langle C, \Gamma \rangle = & \sum_{i=1}^n \Big( (x_i-0)^2  \Gamma_{i,1} + (x_i-1)^2 \Gamma_{i,2} \Big)\\
    = & \sum_{i=1}^n \Big( x_i^2 (\Gamma_{i,1} + \Gamma_{i,2}) + \Gamma_{i,2} - 2x_i \Gamma_{i,2}\Big)\\
    = & \frac{1}{n} \sum_{i=1}^n x_i^2 + \frac{n-k}{n} -2 \sum_{i=1}^n x_i \Gamma_{i,2}.
\end{align*}
Therefore, to minimize $\langle C, \Gamma \rangle$, it suffices to maximize $\sum_{i=1}^n x_i \Gamma_{i,2}$. It is straightforward to check $$\sum_{i=1}^n \Gamma_{i, 2} = \frac{n-k}{n} \quad \textrm{and} \quad \Gamma_{i, 2} \leq  \frac{1}{n}$$ for any $i = 1, \dots, n$. Hence, maximizing $\sum_{i=1}^n x_i \Gamma_{i,2}$ is essentially selecting the largest $n-K$ elements from $\cX$, and the maximum is attained at 
\begin{align*}
    \Gamma^*_{\sigma_i,2}= \begin{cases}
    0, & {\rm ~if~} i\leq k, \\
    1/n, & {\rm ~if~} k+1 \leq i\leq n.
    \end{cases}
\end{align*}
The constraint $\Gamma \bm{1}_m = \mu$ further implies that $\Gamma^*_{i, 1}$ satisfies \eqref{eq:gamma1}. Thus, $A$ can be parameterized as $A=n\Gamma^*\cdot [1,0]^\top$.
\end{proof}

\begin{figure}
    \centering
    \includegraphics[width=0.55\linewidth]{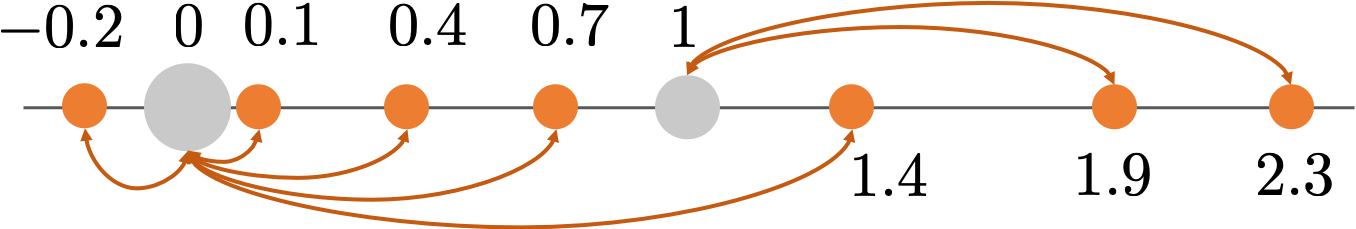}
    \caption{\label{fig:illu} Illustration of the optimal transport plan with input $\cX=[0.4,0.7,2.3,1.9,-0.2,1.4,0.1]^\top$ and $k=5$. Here, we set $\nu=[\frac{5}{7}, \frac{2}{7}]^\top$. In this way, $5$ of the $7$ scores, i.e., $\{0.4,0.7,-0.2,1.4,0.1\}$, would align with $0$, while $\{2.3, 1.9\}$ align with $1$. }
\end{figure}

Figure \ref{fig:illu} illustrates the corresponding optimal transport plan for parameterizing the top-$5$ operator applied to a set of $7$ elements. As can be seen, the mass from the $5$ closest points is transported to $0$, and meanwhile the mass from the $2$ remaining points is transported to $1$. Therefore, the optimal transport plan exactly indicates the top-$5$ elements.

\subsection{Smoothing by Entropy Regularization}

We next rule out the discontinuity of \eqref{eq:kanto} to obtain a smoothed approximation to the standard top-$k$ operator.



Specifically, we employ entropy regularization to the OT problem \eqref{eq:kanto}:
\begin{align}
\label{eq:reg_ot}
 \Gamma^{*, \epsilon} = \argmin_{\Gamma \geq 0}  \langle C, \Gamma \rangle + \epsilon H(\Gamma), \quad{\rm s.t.,}\quad \Gamma \bm{1}_m = \mu, \Gamma^\top \bm{1}_n = \nu, 
\end{align}
where $h(\Gamma)=\sum_{i,j}\Gamma_{ij}\log \Gamma_{ij}$ is the entropy regularizer. We define $A^{\epsilon} = n \Gamma^{*, \epsilon} \cdot [0, 1]^\top$ as a smoothed counterpart of output $A$ in the standard top-$k$ operator. Accordingly, SOFT top-$k$ operator is defined as the mapping from $\cX$ to $A^\epsilon$. We show that the Jacobian matrix of SOFT top-$k$ operator exists and is nonzero in the following theorem.  


\begin{figure}[t]
\centering     
\subfigure[$\epsilon=10^{-3}$]{\label{fig:a}\includegraphics[height=0.13\textwidth]{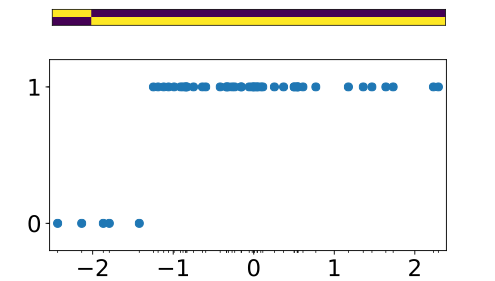}}
\subfigure[$\epsilon=5\times10^{-3}$]{\label{fig:b}\includegraphics[height=0.13\textwidth]{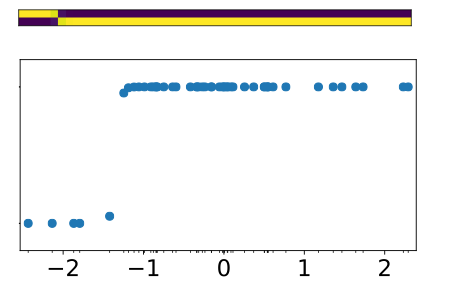}}
\subfigure[$\epsilon=10^{-2}$]{\label{fig:a}\includegraphics[height=0.13\textwidth]{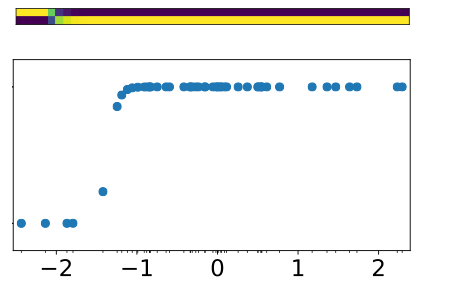}}
\subfigure[$\epsilon=5\times10^{-2}$]{\label{fig:b}\includegraphics[height=0.13\textwidth]{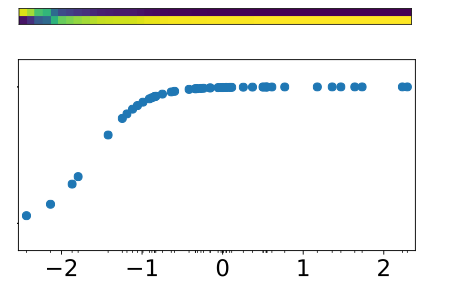}}
\subfigure{\label{fig:b}\includegraphics[height=0.14\textwidth]{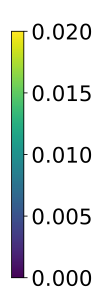}}
\caption{Color maps of $\Gamma^{\epsilon}$ (upper) and the corresponding scatter plots of values in $A^{\epsilon}$ (lower), where $\mathcal{X}$ contains $50$ standard Gaussian samples, and $K=5$. The scatter plots show the correspondence of the input $\mathcal{X}$ and output $A^{\epsilon}$.}
\end{figure}

\begin{theorem}
For any $\epsilon>0$, SOFT top-$k$ operator: $\cX \mapsto A^\epsilon$ is differentiable, as long as the cost $C_{ij}$ is differentiable with respect to $x_i$ for any $i, j$. Moreover, the Jacobian matrix of SOFT top-$k$ operator always has a nonzero entry for any $\cX \in\mathbb{R}^n$.
\end{theorem}

The proof can be found in Appendix \ref{sec:theorem_proof}. We remark that the entropic OT \eqref{eq:reg_ot} is computationally more friendly, since it allows the usage of first-order algorithms \citep{cuturi2013sinkhorn}. 


The Entropic OT introduces bias to the SOFT top-$k$ operator. The following theorem shows that such a bias can be effectively controlled.

\begin{theorem}
Given a distinct sequence $\cX$ and its sorting permutation $\sigma$, with Euclidean square cost function, for the proposed top-$k$ solver we have 
\begin{align*}
    \|\Gamma^{*, \epsilon} - \Gamma^{*}\|_{\textrm{F}} \leq \frac{\epsilon (\ln n + \ln 2)}{n(x_{\sigma_{k+1}} -x_{\sigma_{k}}) }.
\end{align*}
\end{theorem}
Therefore, with a small enough $\epsilon$, the output vector $A^{\epsilon}$ can well approximate $A$, especially when there is a large gap between $x_{\sigma_{k}}$ and $x_{\sigma_{k+1}}$. Besides, Theorem \ref{thm:lm2} suggests a trade-off between the bias and regularization of SOFT top-$k$ operator. See Section \ref{sec:diss} for a detailed discussion.


\subsection{Sorted SOFT Top-$k$ Operator}\label{sorted-top-k-soft}

In some applications like beam search, we not only need to distinguish the top-$k$ elements, but also sort the top-$k$ elements. For example, in image retrieval \citep{gordo2016deep}, the retrieved $k$ images are expected to be sorted. We show that our proposed SOFT top-$k$ operator can be extended to the sorted SOFT top-$k$ operator.

\begin{figure}
    \centering
    \includegraphics[width=0.55\linewidth]{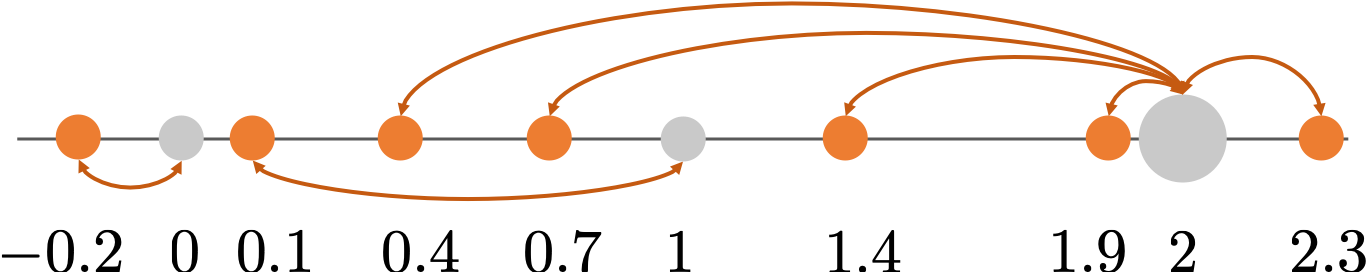}
    \caption{\label{fig:illu_sorted} Illustration of the optimal transport plan for sorted top-$k$ with input $\cX=[0.4,0.7,2.3,1.9,-0.2,1.4,0.1]^\top$ and $K=2$. Here, we set $\nu=[\frac{1}{7}, \frac{1}{7}, \frac{5}{7}]^\top$ and $\cB=[0,1,2]^\top$. In this way, the smallest score $-0.2$ aligns with $0$, the second smallest score $0.1$ aligns with $1$, and the rest of the scores align with $2$. }
\end{figure}

Analogous to the derivation of the SOFT top-$k$ operator, we first parameterize the sorted top-$k$ operator in terms of an OT problem. Specifically, we keep $\cA = \cX$ and $\mu = \bm{1}_n / n$ and set 
\begin{align*}
     \cB = [0,1,2, \cdots, k]^\top,\quad \textrm{and}  \quad \nu = [1/n, \cdots, 1/n, (n-k)/n]^\top.
\end{align*}

One can check that the optimal transport plan of the above OT problem transports the smallest element in $\cA$ to $0$ in $\cB$, the second smallest element to $1$, and so on so forth. This in turn yields the sorted top-$k$ elements. Figure \ref{fig:illu_sorted} illustrates the sorted top-$2$ operator and its corresponding optimal transport plan.

The sorted SOFT top-$k$ operator is obtained similarly to SOFT top-$k$ operator by solving the entropy regularized OT problem. We can show that the sorted SOFT top-$k$ operator is differentiable and the bias can be properly controlled.


\section{Efficient Implementation}
\label{sec:method2}

We now present our implementation of SOFT top-$k$ operator, which consists of 1) computing $A^\epsilon$ from $\cX$ and 2) computing the Jacobian matrix of $A^\epsilon$ with respect to $\cX$. We refer to 1) as the forward pass and 2) as the backward pass. 

\noindent{\bf Forward Pass.} The forward pass from $\cX$ to $A^\epsilon$ can be efficiently computed using Sinkhorn algorithm. Specifically, we 
run iterative Bregman projections \citep{benamou2015iterative}, where at the $\ell$-th iteration, we update
 \[
 p^{(\ell+1)} = \frac{\mu}{Gq^{(\ell)}}, \quad q^{(\ell+1)} = \frac{\nu}{G^\top p^{(\ell+1)}}.
 \]
Here the division is entrywise, $q^{(0)} = \bm{1}_2/2$, and $G \in \RR^{n \times m}$ with $G_{ij} = e^{-\frac{C_{ij}}{\epsilon}}$. Denote $p^*$ and $q^*$ as the stationary point of the Bregman projections. The optimal transport plan $\Gamma^{*, \epsilon}$ can be otained by $\Gamma^{*, \epsilon}_{ij} = p_i^* G_{ij} q_j^*$.
The algorithm is summarized in Algorithm \ref{alg:topk}.

\begin{algorithm}
\caption{\label{alg:topk} SOFT Top-$k$}
\begin{algorithmic} 
\REQUIRE $\cX = [x_i]_{i=1}^n, k, \epsilon, L$
\STATE $\cY = [y_1, y_2]^\top = [0, 1]^\top$
\STATE $\mu = \bm{1}_n/n, \nu = [k/n, (n-K)/n]^\top$
\STATE  $C_{ij} = |x_i-y_j|^2, G_{ij} = e^{-\frac{C_{ij}}{\epsilon}}, q = \bm{1}_2/2$
\FOR{$l = 1, \cdots, L$}
\STATE $p = {\mu}/(Gq), q = \nu/(G^\top p) $
\ENDFOR
\STATE $\Gamma = {\rm diag}(p)\odot G \odot {\rm diag}(q)$
\STATE $A^\epsilon = n \Gamma \cdot [0,1]^\top $
\end{algorithmic}
\end{algorithm}

\noindent{\bf Backward Pass.} Given $A^{\epsilon}$, we compute the Jacobian matrix $\frac{d A^{\epsilon}}{d \cX}$ using implicit differentiation and differentiable programming techinques. Specifically, the Lagrangian function of Problem \eqref{eq:reg_ot} is
\begin{align*}
    \mathcal{L} = \langle C, \Gamma \rangle - \xi^\top (\Gamma \bm{1}_m - \mu) - \zeta^\top (\Gamma^\top \bm{1}_n - \nu) +\epsilon H(\Gamma),
\end{align*}
where $\xi$ and $\zeta$ are dual variables.
The KKT condition implies that the optimal solution $\Gamma^{*, \epsilon}$ can be formulated using the optimal dual variables $\xi^*$ and $\zeta^*$ as (Sinkhorn’s scaling theorem, \citet{sinkhorn1967concerning}), 
\begin{align} \label{eq:pri_dual}
    \Gamma^{*, \epsilon} = {\rm diag}(e^{\frac{\xi^*}{\epsilon}})e^{-\frac{C}{\epsilon}}{\rm diag}(e^{\frac{\zeta^*}{\epsilon}}).
\end{align}
Substituting \eqref{eq:pri_dual} into the Lagrangian function, we obtain 
\begin{align*}
    \mathcal{L}(\xi^*,\zeta^*; C) = (\xi^*)^\top \mu + (\zeta^*)^\top \nu -\epsilon \sum_{i,j=1}^{n,m} e^{-\frac{C_{ij}-\xi^*_i-\zeta^*_j}{\epsilon}}.
\end{align*}
We now compute the gradient of $\xi^*$ and $\zeta^*$ with respect to $C$, such that we can obtain $d\Gamma^{*, \epsilon}/dC$ by the chain rule applied to \eqref{eq:pri_dual}.
Denote $\omega^* = [(\xi^*)^\top, (\zeta^*)^\top]^\top$, and $\phi(\omega^*; C) = {\partial \mathcal{L}(\omega^*; C)}/{\partial \omega^*}$.
At the optimal dual variable $\omega^*$, the KKT condition immediately yields
\begin{align*}
    \phi(\omega^*; C) \equiv 0.
\end{align*}
By the chain rule, we have
\begin{align*}
    \frac{d \phi(\omega^*; C)}{d C} = \frac{\partial \phi(\omega^*; C)}{\partial C} + \frac{\partial \phi(\omega^*; C)}{\partial \omega^*} \frac{d\omega^*}{d C} = 0.
\end{align*}
Rerranging terms, we obtain 
\begin{align} \label{eq:omega_C}
    \frac{d\omega^*}{d C} = - \left(\frac{\partial \phi(\omega^*; C)}{\partial \omega^*}\right)^{-1} \frac{\partial \phi(\omega^*; C)}{\partial C}.
\end{align}
Combining \eqref{eq:pri_dual}, \eqref{eq:omega_C}, $C_{ij}=(x_i-y_j)^2$, and $A^\epsilon =n\Gamma^{*, \epsilon}\cdot [1,0]^\top$, the Jacobian matrix $dA^\epsilon /d\cX$ can then be derived using the chain rule again.

The detailed derivation and the corresponding algorithm for computing the Jacobian matrix can be found in Appendix \ref{sec:dev_gradient}. The time and space complexity of the derived algorithm is $\cO(n)$ and $\cO(kn)$ for top-$k$ and sorted top-$k$ operators, respectively. We also include a Pytorch \citep{paszke2017automatic} implementation of the forward and backward pass in Appendix \ref{sec:dev_gradient} by extending the \texttt{autograd} automatic differentiation package. 




\section{$k$-NN for Image Classification}
\label{sec:app}







\begin{figure}
    \centering
    \includegraphics[width=0.6\linewidth]{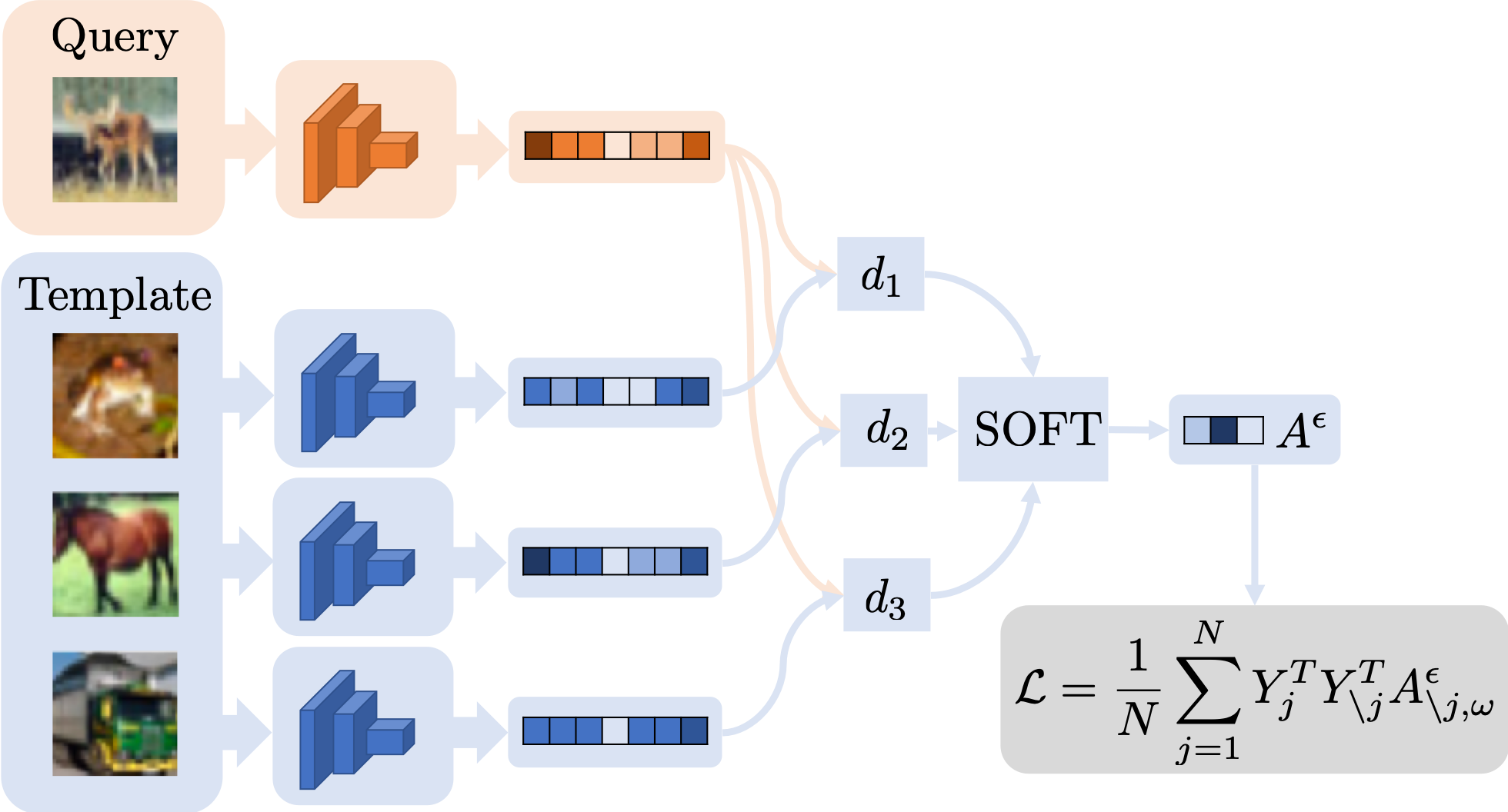}
    \caption{\label{fig:knn} Illustration of the entire forward pass of $k$NN.}
\end{figure}

The proposed SOFT top-$k$ operator enables us to train an end-to-end neural network-based $k$NN classifier. Specifically, we receive training samples $\{Z_i, y_i\}_{i=1}^N$ with $Z_i$ being the input data and $y_i \in \{1, \dots, M\}$ the label from $M$ classes. During the training, for an input data $Z_j$ (also known as the query sample), we associate a loss as follows. Denote $Z_{\setminus j}$ as all the input data excluding $Z_j$ (also known as the template samples). We use a neural network to extract features from all the input data, and measure the pairwise Euclidean distances between the extracted features of $Z_{\setminus j}$ and that of $Z_j$. Denote $\cX_{\setminus j, \theta}$ as the collection of these pairwise distances, i.e., 
\begin{align*}
 \cX_{\setminus j, \theta} = \{\norm{f_\theta(Z_1) - f_\theta(Z_{j})}_2, ... , \norm{f_\theta(Z_{j-1}) - f_\theta(Z_j)}_2,  \norm{f_\theta(Z_{j+1}) - f_\theta(Z_j)}_2, ... , \norm{f_\theta(Z_N) - f_\theta(Z_j)}_2\},
\end{align*}
where $f_\theta$ is the neural network parameterized by $\theta$, and the subscript of $\cX$ emphasizes its dependence on $\theta$.

Next, we apply SOFT top-$k$ operator to $\cX_{\setminus j, \omega}$, and the returned vector is denoted by $A^\epsilon_{\setminus j, \theta}$. Let $Y_{\setminus j} \in \RR^{M \times (N-1)}$ be the matrix by concatenating the one-hot encoding of labels $y_i$ for $i \neq j$ as columns, and $Y_j \in \RR^M$ the one-hot encoding of the label $y_j$. The loss of $Z_j$ is defined as
\begin{align*}
\ell(Z_j, y_j) = Y_j^\top Y_{\setminus j}^\top A^\epsilon_{\setminus j, \theta}.
\end{align*}
Consequently, the training loss is
\begin{align*}
\cL(\{Z_j, y_j\}_{j=1}^N) = \frac{1}{N} \sum_{j = 1}^N \ell(Z_j, y_j) = \frac{1}{N} \sum_{j=1}^N Y_j^\top Y_{\setminus j}^\top A^\epsilon_{\setminus j, \theta}.
\end{align*}
Recall that the Jacobian matrix of $A^\epsilon_{\setminus j, \theta}$ exists and has no zero entries. This allows us to utilize stochastic gradient descent algorithms to update $\theta$ in the neural network. 
Moreover, since $N$ is often large, to ease the computation, we randomly sample a batch of samples to compute the stochastic gradient at each iteration.

In the prediction stage, we use all the training samples to obtain a predicted label of a query sample. Specifically, we feed the query sample into the neural network to extract its features, and compute pairwise Euclidean distances to all the training samples. We then run the standard $k$NN algorithm \citep{hastie2009elements} to obtain the predicted label. 



\subsection{Experiment}

We evaluate the performance of the proposed neural network-based $k$NN classifier on two benchmark datasets: MNIST dataset of handwritten digits \citep{lecun1998gradient} and the CIFAR-10 dataset of natural images \citep{krizhevsky2009learning} with the canonical splits for training and testing without data augmentation.  We adopt the coefficient of entropy regularizer $\epsilon=10^{-3}$ for MNIST dataset and $\epsilon=10^{-5}$ for CIFAR-10 dataset. Detailed settings of the model and training procedure are deferred to Appendix \ref{sec:exp_setting}.

\noindent{\bf Baselines.} We consider several baselines: 
\begin{itemize}[topsep=0pt,leftmargin=*,nolistsep,nosep]
    \item [1.] Standard $k$NN method. 
    \item [2.] Two-stage training methods: we first extract the features of the images, and then perform $k$NN on the features. The feature is extracted using Principle Component Analysis (PCA, top-$50$ principle components is adopted), autoencoder (AE), or a pretrained Convolutional Neural Network (CNN) using the Cross-Entropy (CE) loss. 
    \item[3.] Differentiable \textit{ranking} + $k$NN: This includes NeuralSort \citep{grover2019stochastic} and \citet{cuturi2019differentiable}. \citet{cuturi2019differentiable} is not directly applicable, which requires some adaptation (see Appendix \ref{sec:exp_setting}).
    \item[4.] Stochastic $k$NN with Gumbel top-$k$ relaxation \citep{xie2019reparameterizable}: The model is referred as RelaxSubSample. 
    \item[5.] Softmax Augmentation for smoothed top-$k$ operation: A combination of $k$ softmax operation is used to replace the top-$k$ operator. Specifically, we recursively perform softmax on $\cX$ for $k$ times (Similar idea appears in \citet{plotz2018neural} and \citet{goyal2018continuous}). At the $k$-th iteration, we mask the top-$(k-1)$ entries with negative infinity. 
    \item[6.] CNNs trained with CE without any top-$k$ component\footnote{Our implementation is based on github.com/pytorch/vision.git}.
\end{itemize}

For the pretrained CNN and CNN trained with CE, we adopt identical neural networks as our method.
 
\noindent{\bf Results.} We report the classification accuracies on the standard test sets in Table \ref{tab:knn_result}. On both datasets, the SOFT $k$NN classifier achieves comparable or better accuracies. 


\begin{table}[htb!]
\centering
\vspace{-0.1in}
\caption{\label{tab:knn_result} Classification accuracy of kNN.}
\begin{tabular}{lll}
\hline
Algorithm      & MNIST    & CIFAR10  \\ \hline
$k$NN            & $97.2\%$ & $35.4\%$ \\
$k$NN+PCA        & $97.6\%$ & $40.9\%$ \\
$k$NN+AE         & $97.6\%$ & $44.2\%$ \\
$k$NN+pretrained CNN & $98.4\%$ & $91.1\%$ \\
RelaxSubSample & $99.3\%$ & $90.1\%$ \\
$k$NN+NeuralSort & $\bm{99.5}\%$ & $90.7\%$ \\
$k$NN+\citet{cuturi2019differentiable} & $99.0\%$ & $84.8\%$\\
$k$NN+Softmax $k$ times & $99.3\%$ & $92.2\%$ \\
CE+CNN \citep{he2016deep}   & $99.0\%$ & $91.3\%$ \\
$k$NN+\textbf{SOFT Top-$k$}   & $99.4\%$ & $\bm{92.6}\%$ \\  \hline
\vspace{-0.1in}
\end{tabular}
\vspace{-0.1in}
\end{table}


\section{Beam Search for Machine Translation}
\label{sec:app2}

Beam search is a popular method for the \textit{inference} of Neural Language Generation (NLG) models, e.g., machine translation models. Here, we propose to incorporate beam search into the \textit{training} procedure based on SOFT top-$k$ operator.


\subsection{Misalignment between Training and Inference}
\label{sec:beam_search_inf}

Denote the predicted sequence as $y=[y^{(1)}, \cdots, y^{(T)}]$, and the vocabularies as $\{z_1, \cdots, z_V\}$. Consider a recurrent network based NLG model. The output of the model at the $t$-th decoding step  is a probability simplex $[\mathbb{P}(y^{(t)}=z_i|h^{(t)}]_{i=1}^V$, where $h^{(t)}$ is the hidden state associated with the sequence $y^{(1:t)}=[y^{(1)},...,y^{(t)}]$.

Beam search recursively keeps the sequences with the $k$ largest likelihoods, and discards the rest. Specifically, at the $(t+1)$-th decoding step, we have $k$ sequences $\tilde{y}^{(1:t),i}$'s obtained at the $t$-th step, where $i=1,...,k$ indexes the sequences. The likelihood of $\tilde{y}^{(1:t),i}$ is denoted by $\cL_{\rm s}(\tilde{y}^{(1:t),i})$. 
We then select the next $k$ sequences by varying $i = 1, \dots, k$ and $j = 1, \dots, V$:
\begin{align*}
    \{\tilde{y}^{(1:t+1),\ell}\}_{\ell=1}^k = \arg \mathrm{top\textrm{-}k}_{[\tilde{y}^{(1:t),i}, z_j]} \cL_{\rm s}([ \tilde{y}^{(1:t),i},z_j]).
\end{align*}
where $\cL_{\rm s}([\tilde{y}^{(1:t),i},z_j])$ is the  likelihood of the sequence appending $z_j$ to $\tilde{y}^{(1:t),i}$ defined as
\begin{align} \label{eq:likelihood}
    \cL_{\rm s}([\tilde{y}^{(1:t),i},z_j])\!=\! \mathbb{P}(y^{(t+1)}\!=\!z_j|h^{(t+1),i})\cL_{\rm s}(\tilde{y}^{(1:t),i}),
\end{align}
and $h^{(t+1),i}$ is the hidden state generated from $\tilde{y}^{(1:t),i}$. Note that $z_j$'s and $\tilde{y}^{(1:t),i}$'s together yield $Vk$ choices.
Here we abuse the notation: $\tilde{y}^{(1:t+1),\ell}$ denotes the $\ell$-th selected sequence at the $(t+1)$-th decoding step, and is not necessarily related to $\tilde{y}^{(1:t),i}$ at the $t$-th decoding step, even if $i=\ell$.



For $t=1$, we set $\tilde{y}^{(1)} = z_{\rm s}$ as the start token, $\mathcal{L}_{\rm s}(y^{(1)})=1$, and $h^{(1)} = h_{\rm e}$ as the output of the encoder.
We repeat the above procedure, until the end token is selected or the pre-specified max length is reached. At last, we select the sequence $y^{(1:T),*}$ with the largest likelihood as the predicted sequence.


Moreover, the most popular training procedure for NLG models directly uses the so-called ``\textit{teacher forcing}'' framework. As the ground truth of the target sequence (i.e., gold sequence) $\bar{y} = [\bar{y}^{(1)}, \cdots, \bar{y}^{(T)}]$ is provided at the training stage,  we can directly maximize the likelihood 
\begin{align}\label{eq:tf}
\mathcal{L}_{\rm tf} = \prod_{t=1}^T \mathbb{P}(y^{(t)}=\bar{y}^{(t)}|h^{(t)}(\bar{y}^{(1:t\textrm{-}1)})).
\end{align}
As can be seen, such a training framework only involve the gold sequence, and cannot take the uncertainty of the recursive exploration of the beam search into consideration. Therefore, it yields a misalignment between model training and inference \citep{bengio2015scheduled}, which is also referred as \textit{exposure bias} \citep{wiseman2016sequence}.

\subsection{Differential Beam Search with Sorted SOFT Top-$k$}

To mitigate the aforementioned misalignment, we propose to integrate beam search into the training procedure, where the top-$k$ operator in the beam search algorithm is replaced with our proposed sorted SOFT top-$k$ operator proposed in Section \ref{sorted-top-k-soft}. 

Specifically, at the $(t+1)$-th decoding step, we have $k$ sequences denoted by $E^{(1:t),i}$, where $i=1,...,k$ indexes the sequences. Here $E^{(1:t),i}$ consists of a sequence of $D$-dimensional vectors, where $D$ is the embedding dimension. We are not using the tokens, and the reason behind will be explained later. Let $\tilde{h}^{(t),i}$ denote the hidden state generated from $E^{(1:t),i}$. We then consider
\begin{align*}
    \mathcal{X}^{(t)} = \{-\mathcal{L}_{\textrm{s}}([E^{(1:t),i},w_j]), j=1,..., V, ~i=1, ..., k\},
\end{align*}
where $\mathcal{L}_{\textrm{s}}(\cdot)$ is defined analogously to \eqref{eq:likelihood}, and $w_j\in\RR^D$ is the embedding of token $z_j$. 

Recall that $\epsilon$ is the smoothing parameter. We then apply the sorted SOFT top-$k$ operator to $\mathcal{X}^{(t)}$ to obtain $\{E^{(1:t+1),\ell}\}_{\ell=1}^k$, which are $k$ sequences with the largest likelihoods. More precisely, the sorted SOFT top-$k$ operator yields an output tensor $A^{(t),\epsilon}\in\RR^{V\times k\times k}$, where $A^{(t),\epsilon}_{ji,\ell}$ denotes the smoothed indicator of whether $[E^{(1:t),i},w_j]$ has a rank $\ell$. We then obtain
\begin{align}\label{eq:emd}
E^{(1:t+1),\ell} = \Big[E^{(1:t),r},\sum_{j=1}^V\sum_{i=1}^kA^{(t),\epsilon}_{ji,\ell}w_j\Big],
\end{align}
where $r$ denotes the index $i$ (for $E^{(1:t),i}$'s) associated with the index $\ell$ (for $E^{(1:t+1),\ell}$'s). This is why we use vector representations instead of tokens: this allows us to compute  $E^{(t+1),\ell}$ as a weighted sum of all the word embeddings $[w_j]_{j=1}^V$, instead of discarding the un-selected words.

Accordingly, we generate the $k$ hidden states for the $(t+1)$-th decoding step:
\begin{align}\label{eq:hid}
    \tilde{h}^{(t),\ell} =   \sum_{j=1}^V \sum_{i=1}^{k} A^{(t),\epsilon}_{ji,\ell}h^{(t),i},
\end{align}
where $h^{(t),i}$ is the intermediate hidden state generated by the decoder based on $E^{(1:t),i}$.


After decoding, we select the sequence with largest likelihood $E^{(1:T),*}$, and maximize the likelihood as follows,
\begin{align*} 
    \mathcal{L}_{{\rm SOFT}} = \prod_{t=1}^T & \mathbb{P}(y^{(t)}=\bar{y}^{(t)}|\tilde{h}^{(t\textrm{-}1),*}(E^{(1:t\textrm{-}1),*})).
\end{align*}
We provide the sketch of training procedure in Algorithm \ref{alg:beam_search2}, where we denote logit$^{(t),i}$ as $[\log \mathbb{P}(y^{(t+1)}=\omega_j|\tilde{h}^{(t),i}(E^{(1:t),i}))]_{j=1}^V$, which is part of the output of the decoder. More technical details (e.g., backtracking algorithm for finding the index $r$ in \eqref{eq:emd}) are provided in Appendix \ref{sec:exp_setting}.

\begin{algorithm}[htb!]
\caption{\label{alg:beam_search2} Beam search training with SOFT Top-$k$}
\begin{algorithmic} 
\REQUIRE Input sequence $s$, target sequence $\bar{y}$; embedding matrix $W\in \mathbb{R}^{V\times D}$; max length $T$; $k$;  regularization coefficient $\epsilon$; number of Sinkhorn iteration $L$
\STATE $\tilde{h}^{(1)}_i = h_{\rm e} = $ Encoder$(s)$,
$E^{(1),i}=w_{\rm s}$ 
\FOR{$t=1,\cdots,T-1$}
\FOR{$i=1, \cdots, k$}
\STATE \hspace{-0.15in} logit$^{(t),i}, h^{(t),i} =$ Decoder$(E^{(t),i}, \tilde{h}^{(t),i})$\\
\STATE \hspace{-0.15in} $\log \mathcal{L}_{\rm s}([E^{(1:t),i},w_j])=\log \mathcal{L}_{\rm s}( E^{(1:t),i})+$logit$^{(t),i}_{j}$\\
\STATE \hspace{-0.15in} $\mathcal{X}^{(t)} = \{-\log \mathcal{L}_{\rm s}([E^{(1:t),i},w_j])~|~j=1, \cdots, V\}$
\ENDFOR
\STATE $A^{(t), \epsilon}$ = Sorted-SOFT-Top-$k$($\cX^{(t)}, k, \epsilon, L$)
\STATE Compute $E^{(t+1),\ell}$, $\tilde{h}^{(t+1),\ell}$ as in \eqref{eq:emd} and \eqref{eq:hid}
\ENDFOR
\STATE Compute $\nabla\mathcal{L}_{{\rm SOFT}}$ and update the model
\end{algorithmic}
\end{algorithm}

Note that integrating the beam search into training  essentially yields a very large search space for the model, which is not necessarily affordable sometimes. To alleviate this issue, we further propose a hybrid approach by combining the teacher forcing training with beam search-type training. Specifically, we maximize the weighted likelihood defined as follows,
\begin{align*}
    \cL_{\rm final} = \rho \cL_{\rm tf} + (1- \rho)\cL_{\rm SOFT},
\end{align*}
where $\rho\in(0,1)$ is referred to as the ``teaching forcing ratio''. The teaching forcing loss $\cL_{\rm tf}$ can help reduce the search space and improve the overall performance.

\subsection{Experiment}

We evaluate our proposed beam search + sorted SOFT top-$k$ training procedure using WMT2014 English-French dataset. 

\noindent{\bf Settings.} We adopt beam size $5$, teacher forcing ratio $\rho=0.8$, and $\epsilon=10^{-1}$.  For detailed settings of the training procedure, please refer to Appendix \ref{sec:exp_setting}. 

We reproduce the experiment in \citet{bahdanau2014neural}, and run our proposed  training procedure with the identical data pre-processing procedure and the LSTM-based sequence-to-sequence model. Different from \citet{bahdanau2014neural}, here we also preprocess the data with \textit{byte pair encoding} \citep{sennrich2015neural}.  

\noindent{\bf Results.} As shown in Table \ref{tab:beam_result}, the proposed SOFT beam search training procedure achieves an improvement in BLEU score of approximately $0.9$. We also include other LSTM-based models for baseline comparison.



\begin{table}[!htb]
\centering
\caption{\label{tab:beam_result} BLEU scores on WMT'14 with single LSTM model.}
\begin{tabular}{ll}
\hline
Algorithm      & BLEU  \\ \hline
\citet{luong2014addressing}  & $33.10$ \\
\citet{durrani2014edinburgh} & $30.82$ \\
\citet{cho2014learning} & $34.54$ \\
\citet{sutskever2014sequence} & $30.59$ \\
\citet{bahdanau2014neural} & $28.45$ \\
\citet{jean2014using}   & $34.60$ \\
\citet{bahdanau2014neural} (Our implementation)    & $35.38$ \\
{\bf Beam Search + Sorted SOFT Top-k}    & $\bm{36.27}$ \\ \hline
\end{tabular}
\vspace{-0.1in}
\end{table}

\section{Top-$k$ Attention for Machine Translation}
\label{sec:app3}

We apply SOFT top-$k$ operator to yield sparse attention scores. Attention module is an integral part of various natural language processing tasks, allowing modeling of long-term and local dependencies. Specifically, given the vector representations of a source sequence $s=[s_1, \cdots, s_{N}]^\top$ and target sequence $y=[y_1, \cdots, y_{M}]^\top$, we compute the alignment score between $s_i$ and $y_j$ by a compatibility function $f(s_i, y_j)$, e.g., $f(s_i, y_j)=s_i^\top y_j$, which measures the dependency between $s_i$ and $y_j$. A softmax function then transforms the scores $[f(s_i,y_j)]^{N}_{i=1}$ to a sum-to-one weight vector $w_j$ for each $y_j$. The output $o_j$ of this attention module is a weighted sum of $s_i$'s, i.e., $o_j = w_j^\top s.$


The attention module described above is called the soft attention, i.e., the attention scores $w_j$ of $y_j$ is not sparse. This may lead to redundancy of the attention \citep{zhu2018fine, schlemper2019attention}. Empirical results show that hard attention, i.e., enforcing sparsity structures in the score $w_j$'s, yields more appealing performance \citep{shankar2018surprisingly}. Therefore, we propose to replace the softmax operation on $[f(s_i,y_j)]^{N}_{i=1}$ by the standard top-$k$ operator to select the top-$k$ elements. In order for an end-to-end training, we further deploy SOFT top-$k$ operator to substitute the standard top-$k$ operator. Given $[f(s_i,y_j)]^{N}_{i=1}$, the output of SOFT top-$k$ operator is denoted by $A^\epsilon_j$, and the weight vector $w_j$ is now computed as
\begin{align*}
    w_{j} = {\rm softmax}([f(s_1,y_j),\dots, f(s_N, y_j)]^\top + \log A^{\epsilon}_{j}).
\end{align*}
\begin{figure}
    \centering
    \includegraphics[width=0.5\linewidth]{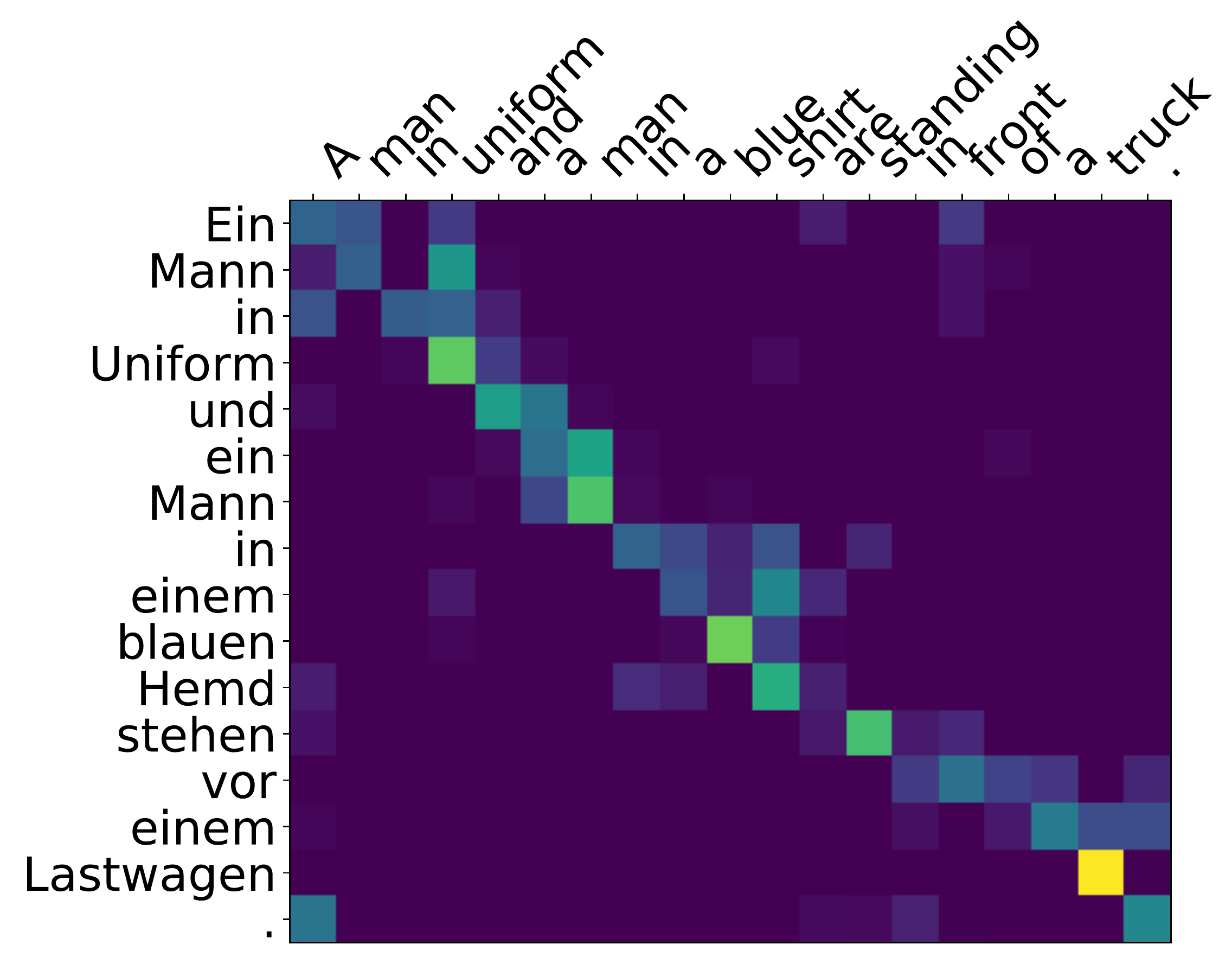}
    \caption{\label{fig:attn} Visualization of the top-$K$ attention. }
\end{figure}
Here $\log$ is the entrywise logarithm. The output $o_j$ of the attention module is computed the same $o_j = w_j^\top s$. 
Such a SOFT top-$k$ attention will promote the top-$k$ elements in $[f(s_i, y_j)]_{i=1}^{N}$ to be even larger than the non-top-$k$ elements, and eventually promote the attention of $y_j$ to focus on $k$ tokens in $s$. 


\subsection{Experiment}

We evaluate the proposed top-$k$ attention on WMT2016 English-German dataset. Our implementation and settings are based on \citet{opennmt}\footnote{Settings on data pre-processing, model, and training procedure is identical to https://opennmt.net/OpenNMT-py/extended.html.}. For a fair comparison, we implement a standard soft attention using the same settings as the baseline. The details are provided in Appendix \ref{sec:exp_setting}.

\noindent{\bf Results.} As shown in Table \ref{tab:attn_result}, the proposed SOFT top-$k$ attention training procedure achieves an improvement in BLEU score of approximately $0.8$. We visualize the top-$k$ attention in Figure \ref{fig:attn}. The attention matrix is sparse, and has a clear semantic meaning -- ``truck" corresponds to ``Lastwagen", ``blue" corresponds to ``blauen", ``standing" corresponds to ``stehen", etc.


\begin{table}[htb!]
\centering
\caption{\label{tab:attn_result} BLEU scores on WMT'16.}
\begin{tabular}{ll}
\hline
Algorithm      & BLEU  \\ \hline
Proposed Top-$k$ Attention   & $\bm{37.30}$ \\
Soft Attention    & $36.54$ \\\hline
\end{tabular}
\vspace{-0.1in}
\end{table}


\section{Related Work}

We parameterize the top-$k$ operator as an optimal transport problem, which shares the same spirit as \citet{cuturi2019differentiable}. Specifically, \citet{cuturi2019differentiable} formulate the ranking and sorting problems as optimal transport problems. Ranking is more complicated than identifying the top-$k$ elements, since one needs to align different ranks to corresponding elements. Therefore, the algorithm complexity per iteration for ranking whole $n$ elements is $\cO(n^2)$. \citet{cuturi2019differentiable} also propose an optimal transport problem for finding the $\tau$-quantile in a set of $n$ elements and the algorithm complexity reduces to $\cO(n)$. Top-$k$ operator essentially finds all the elements more extreme than the $(n-k)/n$-quantile, and our proposed algorithm achieves the same complexity $\cO(n)$ per iteration. The difference is that top-$k$ operator returns the top-$k$ elements in a given input set, while finding a quantile only yields a certain threshold.




Gumbel-Softmax trick \citep{jang2016categorical} can also be utilized to derive a continuous relaxation of the top-$k$ operator. Specifically, \citet{kool2019stochastic} adapted such a trick to sample $k$ elements from $n$ choices, and \citet{xie2019reparameterizable} further applied the trick to stochastic $k$NN, where neural networks are used to approximating the sorting operator. However, as shown in our experiments (see Table \ref{tab:knn_result}), the performance of stochastic $k$NN is not as good as deterministic $k$NN.

Our SOFT beam search training procedure is inspired by several works that incorporate some of the characteristics of beam search into the training procedure \citep{wiseman2016sequence, goyal2018continuous, bengio2015scheduled}. Specifically, \citet{wiseman2016sequence} and \citet{goyal2018continuous} both address the exposure bias issue in beam search.  \citet{wiseman2016sequence} propose a new loss function in terms of the error made during beam search. This mitigates the misalignment of training and testing in beam search. Later, \citet{goyal2018continuous} approximates the top-$k$ operator using $k$ softmax operations (This method is described and compared to our proposed method in \ref{sec:app}). Such an approximation allows an end-to-end training of beam search. In addition, our proposed training loss $\cL_{\textrm{final}}$ is inspired by \citet{bengio2015scheduled}, which combines the teacher forcing training procedure and greedy decoding, i.e., beam search with beam size $1$. 

\section{Discussion}
\label{sec:diss}

\noindent{\bf Relation to automatic differentiation.}
We compute the Jacobian matrix of SOFT top-$k$ operator with respect to its input using the optimal transport plan of the entropic OT problem \eqref{eq:reg_ot} in the backward pass. The optimal transport plan can be obtained by the Sinkhorn algorithm (Algorithm \ref{alg:topk}), which is iterative and each iteration only involves multiplication and addition. Therefore, we can also apply automatic differentiation (auto-diff) to compute the Jacobian matrix. Specifically, we denote $\Gamma_\ell$ as the transport plan at the $t$-th iteration of Sinkhorn algorithm. The update of $\Gamma_\ell$ can be written as $\Gamma_{\ell+1} = \cT(\Gamma_\ell)$, where $\cT$ denotes the update of the Sinkhorn algorithm. In order to apply auto-diff, we need to store all the intermediate states, e.g., $p, q, G$ in each iteration, as defined in Algorithm \ref{alg:topk} at each iteration. This requires a huge memory size proportional to the total number of iterations of the algorithm. In contrast, our backward pass allows us to save memory.

\noindent{\bf Bias and regularization trade-off.}
Theorem \ref{thm:lm2} suggests a trade-off between the regularization and bias of SOFT top-$k$ operator. Specifically, a large $\epsilon$ has a strong smoothing effect on the entropic OT problem, and the corresponding entries of the Jacobian matrix are neither too large nor too small. This eases the end-to-end training process. However, the bias of SOFT top-$k$ operator is large, which can deteriorate the model performance. On the contrary, a smaller $\epsilon$ ensures a smaller bias. Yet the SOFT top-$k$ operator is less smooth, which in turn makes the end-to-end training less efficient.

\begin{figure}[!htb]
    \centering
    \includegraphics[width = 0.28\textwidth, height=0.2\textwidth] {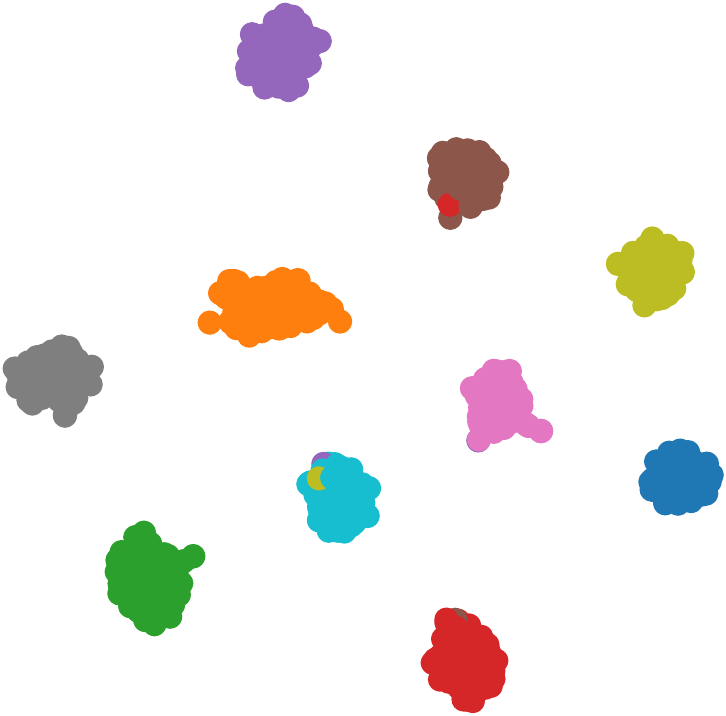}
    \caption{Visualization of the MNIST data based on features extracted by the neural network-based $k$-NN classifier trained by our proposed method in Section \ref{sec:app}.}
    \label{fig:my_label}
    \vspace{-2mm}
\end{figure}
On the other hand, the bias of SOFT top-$k$ operator also depends on the gap between $x_{\sigma_{k+1}}$ and $x_{\sigma_k}$. In fact, such a gap can be viewed as the signal strength of the problem. A large gap implies that the top-$k$ elements are clearly distinguished from the rest of the elements. Therefore, the bias is expected to be small since the problem is relatively easy. Moreover, in real applications such as neural network-based $k$NN classification, the end-to-end training process promotes neural networks to extract features that exhibit a large gap (as illustrated in Figure \ref{fig:my_label}). Hence, the bias of SOFT top-$k$ operator can be well controlled in practice.

\bibliography{reference}
\bibliographystyle{icml2019}

\newpage
\onecolumn
\appendix

\section{Theoretical Guarantees}
\label{sec:theorem_proof}

First, we show that after adding entropy regularization the problem is differentiable.


\noindent{\bf Theorem 1.}
For any $\epsilon>0$, SOFT top-$k$ operator: $\cX \mapsto A^\epsilon$ is differentiable, as long as the cost $C_{ij}$ is differentiable with respect to $x_i$ for any $i, j$. Moreover, the Jacobian matrix of SOFT top-$k$ operator always has a nonzero entry for any $\cX \in\mathbb{R}^n$.

\begin{proof} 
We first prove the differentiability. This part of proof mirrors the proof in \citet{luise2018differential}.
By Sinkhorn's scaling theorem, 
\begin{align*}
        \Gamma^{*, \epsilon} = {\rm diag}(e^{\frac{\xi^*}{\epsilon}})e^{-\frac{C}{\epsilon}}{\rm diag}(e^{\frac{\zeta^*}{\epsilon}}).
    \end{align*}
Therefore, since $C_{ij}$ is differentiable, $\Gamma^{*, \epsilon}$ is differentiable if
$(\xi^*, \zeta^*)$ is differentiable as a function of input scores $X$.

Let us set 
\begin{align*}
        \mathcal{L}(\xi, \zeta; \mu, \nu, C) = \xi^T \mu + \zeta^T \nu -\epsilon \sum_{i,j=1}^{n,m} e^{-\frac{C_{ij}-\xi_i-\zeta_j}{\epsilon}}.
\end{align*} 
and recall that $(\xi^*, \zeta^*) = \argmax_{\xi, \zeta} L(\xi, \zeta; \mu, \nu, C)$. The differentiability of $(\xi^*, \zeta^*)$ is proved using the Implicit Function theorem and follows from the differentiability and strong convexity in $(\xi^*, \zeta^*)$ of the function $\mathcal{L}$. 

Now we prove that $d A^{\epsilon}/d x_\ell$ always has a nonzero entry for $l=1, \cdots, n$. First, we prove that for any $\ell\in \{1, \cdots, n \}$, $d \Gamma^{*, \epsilon}/d x_{\ell}$ always has a nonzero entry. We will prove it by contradiction.  Specifically, the KKT conditions for the stationarity are as follows
\begin{align*}
\xi_i^* + \zeta_j^* =(x_i-y_j)^2 - \epsilon \log \Gamma^{*, \epsilon}_{ij}, \quad \forall i=1, \cdots, n, j=1, \cdots,m.
\end{align*} 
If we view the above formula as a linear equation set of the dual variables, it has $nm$ equations and $m+n$ variables. Therefore, there are $nm-m-n$ redundant equations.
Suppose one of the scores $x_\ell$, has an infinitesimal change $\delta x_\ell$. Assuming $\Gamma^{*, \epsilon}$ does not change, we have a new set of linear equations,
\begin{align*}
& \xi_i^* + \zeta_j^* =(x_i-y_j)^2 - \epsilon \log \Gamma^{*, \epsilon}_{ij}, \quad \forall i\neq \ell, \\
& \xi_{\ell}^* + \zeta_{j}^* =(x_\ell+\delta x_\ell-y_j)^2+\delta C_{\ell j} - \epsilon \log \Gamma^{*, \epsilon}_{\ell j}.
\end{align*} 
Easy to verify that this set of linear equations has no solution. Therefore, there must be at least one entry in $\Gamma^{*, \epsilon}$ has changed. As a result, $d \Gamma^{*, \epsilon}/d x_\ell$ always has a nonzero entry. We denote this entry as $\Gamma^{*, \epsilon}_{i'j'}$.
 Since $\Gamma^{*, \epsilon}_{i'j'}+\Gamma^{*, \epsilon}_{i', 3-j'}=\mu_{i'}$, we have 
\begin{align*}
\frac{d \Gamma^{*, \epsilon}_{i', 3-j'}}{d x_\ell} = -\frac{d \Gamma^{*, \epsilon}_{i'j'}}{d x_\ell} \neq 0.
\end{align*}
Therefore, there must be a nonzero entry in the first column of $d\Gamma^{*, \epsilon}/d x_\ell$. Recall $A^{\epsilon}$ is the first column of $\Gamma^{*, \epsilon}$. As a result, there must be a nonzero entry in $dA^{\epsilon}/d x_\ell$ for any $\ell\in \{1, \cdots, n \}$.


\end{proof}

Second, we would like to know after smoothness relaxation, how much bias is introduced to $A^\epsilon$.

\begin{lemma} \label{thm:lm1} Denote the feasible set of optimal transport problem as $\Delta=\{\Gamma: \Gamma\in [0,1]^{n\times m}, \Gamma \bm{1}_m = \mu, \Gamma \bm{1}_n = \nu\}$. Assume the optimal transport plan is unique. 
Denote $\Gamma^*$ as the optimal transport plan,
\begin{align*}
    \Gamma^* = \argmin_{\Gamma\in \Delta} f(\Gamma) = \argmin_{\Gamma\in \Delta} \langle C, \Gamma \rangle,
\end{align*}
and $\Gamma^{*, \epsilon}$ as the entropy regularized transport plan, 
\begin{align*}
    \Gamma^{*, \epsilon}  = \argmin_{\Gamma\in \Delta} f^{\epsilon} (\Gamma) = \argmin_{\Gamma\in \Delta}f(\Gamma)-\epsilon H(\Gamma) = \argmin_{\Gamma\in \Delta} \langle C, \Gamma \rangle + \epsilon \sum_{i,j} \Gamma_{ij} \ln \Gamma_{ij}.
\end{align*}
We can bound the difference between $\Gamma^*$ and $\Gamma^{*, \epsilon}$ to be
\begin{align*}
\|\Gamma^* - \Gamma^{*, \epsilon}\|_F \leq \epsilon \frac{(\ln n + \ln m)}{B}, 
\end{align*}
where $\|\cdot\|_F$ is the Frobenius norm, and $B$ is a positive constant irrelevant to $\epsilon$.
\end{lemma}

\begin{proof}
Note that $H(\Gamma)$ is the entropy function. Since $0\leq \Gamma_{ij}\leq 1$ and $\sum_{ij} \Gamma_{ij}=1$ for any $\Gamma \in \Delta$, we can view $\Delta$ as the subset of a simplex. Therefore, 
\begin{enumerate}
    \item $H(\Gamma)$ is non-negative.
    \item The maximum of $H(\Gamma)$ in the simplex can be obtained at $\Gamma_{ij}\equiv \frac{1}{nm}$. Therefore the maximum value is $(\ln n + \ln m)$. 
\end{enumerate}
Therefore, $0\leq H(\Gamma) \leq (\ln n + \ln m)$ for  any $\Gamma \in \Delta$. 

Since $H(\Gamma)\geq 0$, we have $f^{\epsilon}(\Gamma)\leq f(\Gamma)$ for any $\Gamma \in \Delta$. As a result, we have $f^{\epsilon}(\Gamma^{*, \epsilon})\leq f(\Gamma^*)$. In other words, we have
\begin{align*}
    \langle C, \Gamma^{*, \epsilon} \rangle  - \epsilon H(\Gamma^{*, \epsilon}) - \langle C, \Gamma^* \rangle  \leq 0.
\end{align*}
Therefore,
\begin{align*}
    \langle C, \Gamma^{*, \epsilon} - \Gamma^* \rangle = \langle C, \Gamma^{*, \epsilon} \rangle  - \langle C, \Gamma^* \rangle  \leq  \epsilon H(\Gamma^{*, \epsilon}) \leq \epsilon (\ln n + \ln m).
\end{align*}

Since the optimal transport problem is a linear optimization problem, $\Gamma^*$ is one of the vertices of $\Delta$. Denote $e_0, e_1, \cdots, e_J$ as the vertices of $\Delta$, and without loss of generality we assume $e_0=\Gamma^*$. Since $\Gamma^{*, \epsilon}\in \Delta$, we can denote $\Gamma^{*, \epsilon} = \sum_{j=0}^J \lambda_j e_j$, where $\lambda_j\geq0$, and $\sum_j \lambda_j =1$. Since $\Gamma^*$ is unique, we have 
\begin{align*}
    \langle C, e_j-e_0 \rangle > 0, \quad \forall j=1, \cdots, J.
\end{align*}
Denote $B_j = \langle C, e_j-e_0 \rangle$. Since the space we are considering is Euclidean space (if we reshape the matrices into vectors), we can write the inner product as
\begin{align*}
    B_j = \langle C, e_j-e_0 \rangle  = \|C\|_F \|e_j-e_0\|_F \cos \theta_{(C, e_j-e_0)} >0.
\end{align*}
So we have $\cos \theta_{(C, e_j-e_0)} >0$. In other words, the angle between $C$ and $e_j-e_0$ is always smaller than $\frac{\pi}{2}$. Therefore, the angle between $C$ and the affine combination of $e_j-e_0$, namely $\sum_{j=0}^J \lambda_j (e_j-e_0)$, is also smaller than $\frac{\pi}{2}$. More specifically, we have
\begin{align*}
    \cos \theta_{(C, \Gamma^{*, \epsilon} - \Gamma^{*})} = \cos \theta_{(C, \sum_{j=0}^J \lambda_j (e_j-e_0))} \geq \min_j \cos \theta_{(C, e_j-e_0)} = \min_j \frac{B_j}{\|C\|_F \|e_j-e_0\|_F }.
\end{align*}
Therefore, we have
\begin{align*}
    \|\Gamma^{*, \epsilon} - \Gamma^{*}\|_F = \frac{\langle C, \Gamma^{*, \epsilon} - \Gamma^* \rangle}{\|C\|_F\cos \theta_{(C, \Gamma^{*, \epsilon} - \Gamma^{*})}} \leq \frac{\epsilon (\ln n + \ln m)}{\|C\|_F \min_j \frac{B_j}{\|C\|_F \|e_j-e_0\|_F }} = \frac{\epsilon (\ln n + \ln m)}{ \min_j \frac{B_j}{\|e_j-e_0\|_F }}.
\end{align*}
Denote $B = \min_j \frac{B_j}{\|e_j-e_0\|_F }$, and we have the conclusion.
\end{proof}

\begin{remark}
In Theorem 1 we restricted the optimal solution to be unique, only for clarity purpose. If it is not unique, similar conclusion holds, except that the proof is more tedious -- instead of divide the vertices into $e_0$ and others, we need to divide it into the vertices that are optimal solutions and the others.
\end{remark}

\begin{lemma} \label{thm:lm2}
At each of the vertices of $\Delta$, the entries of $\Gamma$ are either $0$ or $1/n$ for $\Gamma \in \Delta$. 
\end{lemma}
\begin{proof}
The key idea is to prove by contradiction: If there exist $i,j$ such that $\Gamma_{ij}\in(0, 1/n)$, then $\Gamma$ cannot be a vertex. 

To ease the discussion, we denote $Z=n\Gamma$. We will first prove that the entries of $Z$ are either $0$ or $1$ at the vertices.

Notice that 
\begin{align*}
    & Z_{i,1}+Z_{i,2} = 1, \quad \forall i=1, \cdots, n, \\
    & \sum_i Z_{i,1}=k, \\
    & \sum_i Z_{i,2}=n-k.
\end{align*}
If there exists an entry $Z_{i',j'}\in (0,1)$, then
\begin{enumerate}
    \item $Z_{i',3- j'}\in (0,1)$.
    \item there must exist $i''\neq i'$, such that $Z_{i'',j'}\in (0,1)$. This is because $\sum_{i=1}^n Z_{i,j}$ is an integer, and $Z_{i',j'}$ is not. 
    \item As a result, $Z_{i'', 3-j'}\in (0,1)$.
\end{enumerate}
Therefore, consider $\delta\in(-\min \{1-Z_{i'.j'}, Z_{i',j'}\}, \min \{1-Z_{i'.j'}, Z_{i',j'}\})$ and denote
\begin{align*}
    & \tilde{Z}^{(1)}_{ij} = 
    \begin{cases}
    Z_{i',j'}+\delta, {\rm~if~} i=i', j=j', \\
    Z_{i',3-j'}-\delta, {\rm~if~} i=i', j=3-j', \\
    Z_{i'',j'}-\delta, {\rm~if~} i=i'', j=j', \\
    Z_{i'',3-j'}+\delta, {\rm~if~} i=i'', j=3-j', \\
    Z_{i,j}, {\rm~otherwise.} \\
    \end{cases}\\
    & \tilde{Z}^{(2)}_{ij} = 
    \begin{cases}
    Z_{i',j'}-\delta, {\rm~if~} i=i', j=j', \\
    Z_{i',3-j'}+\delta, {\rm~if~} i=i', j=3-j', \\
    Z_{i'',j'}+\delta, {\rm~if~} i=i'', j=j', \\
    Z_{i'',3-j'}-\delta, {\rm~if~} i=i'', j=3-j', \\
    Z_{i,j}, {\rm~otherwise.} \\
    \end{cases}
\end{align*}
We can easily verify that $\tilde{Z}^{(1)}/n, \tilde{Z}^{(2)}/n \in \Delta$, and also $Z = (\tilde{Z}^{(1)}+\tilde{Z}^{(2)})/2$. Therefore, $Z$ cannot be a vertex. 

\end{proof}

\begin{lemma}\label{thm:lm3}
Given a set of scalar $\{x_1, \cdots, x_n\}$, we sort it to be $\{x_{\sigma_1}, \cdots, x_{\sigma_n}\}$. 
If Euclidean square cost is adopted, $\Gamma^*$ has the following form,
\begin{align*}
    \Gamma^{*}_{ij} = 
    \begin{cases}
       1/n, {\rm ~if~} i=\sigma_\ell, j=1, \ell\leq k\\
       0, {\rm ~if~} i=\sigma_\ell, j=1, k<\ell\leq n \\
       1/n, {\rm ~if~} i=\sigma_\ell, j=2, k<\ell\leq n\\
       0, {\rm ~if~} i=\sigma_\ell, j=2, \ell\leq k \\
    \end{cases}
\end{align*}
And $\min_j \frac{B_j}{\|e_j-e_0\|_F}$ is attained at at a vertex $\Gamma^{**}$, where $\Gamma^{**}_{ij}=\Gamma^{*}_{ij}$ except that the $\sigma_k$-th row and the $\sigma_{k+1}$-th row are swapped. As a result, we have
\begin{align*}
    \min_j \frac{B_j}{\|e_j-e_0\|_F} = n(x_{\sigma_{k+1}} -x_{\sigma_{k}}).
\end{align*}
\end{lemma}
\begin{proof}
From Lemma \ref{thm:lm2}, in each vertex the entries of $\Gamma$ is either $0$ or $1/n$. Also,  $\Gamma^*\in\Delta=\{\Gamma: \Gamma\in [0,1]^{n\times m}, \Gamma \bm{1}_m = \bm{1}_n/n, \Gamma \bm{1}_n = [k/n, (n-k)/n]^\top\}$. Therefore, for the $j$-th vertex, there are $k$ entries with value $1/n$ in the first row of $\Gamma$. Denote the row indices of these $k$ entries as $\mathcal{I}_j$, and $\Omega=\{1, \cdots, n\}$. Then for each vertex we have
\begin{align*}
    & \Gamma_{i,1}=1/n, \quad \forall i\in \mathcal{I}_j \\
    & \Gamma_{i,1}=0, \quad \forall i\in \Omega \backslash \mathcal{I}_j \\
    & \Gamma_{i,2}=1/n, \quad \forall i\in \Omega \backslash \mathcal{I}_j \\
    & \Gamma_{i,2}=0, \quad \forall i\in \mathcal{I}_j.
\end{align*}
Denote $\mathcal{I}^*=\{\sigma_1, \cdots, \sigma_k\}$. We now prove that $\mathcal{I}^*$ corresponds to the optimal solution $\Gamma^*$. This is because for any $j\in \{1, \cdots, J\}$
\begin{align*}
    \Gamma(\mathcal{I}_j) - \Gamma(\mathcal{I}^*) 
    & = \left(\sum_{i\in\mathcal{I}_j} x_i^2 + \sum_{i\in\Omega\backslash\mathcal{I}_j} (x_i-1)^2\right) -\left(\sum_{i\in\mathcal{I}^*} x_i^2 + \sum_{i\in\Omega\backslash\mathcal{I}^*} (x_i-1)^2\right) \\
    & = \left(\sum_{i\in \Omega} x_i^2 - \sum_{i\in\Omega\backslash\mathcal{I}_j} 2x_i + (n-k)\right) - \left(\sum_{i\in \Omega} x_i^2 - \sum_{i\in\Omega\backslash\mathcal{I}^*} 2x_i + (n-k) \right) \\
    & = 2 \left(\sum_{i\in\Omega\backslash\mathcal{I}^*} x_i - \sum_{i\in\Omega\backslash\mathcal{I}_j} x_i \right) \geq 0,
\end{align*}
where the last step is because the elements with indices $\Omega\backslash\mathcal{I}_j$ is the largest $n-k$ elements. Therefore we have $\Gamma(\mathcal{I}^*)=\Gamma^*$. 

Now let's compute $\min_{j\neq 0} B_j/\|e_j-e_0\|$.   Denote set subtraction $\mathcal{A}-\mathcal{B}$ as the set if elements that belongs to $\mathcal{A}$ but do not belong to $\mathcal{B}$, and $|\mathcal{A}|$ as the number of elements in $\mathcal{A}$.
\begin{align*}
    \frac{B_j}{\|e_j-e_0\|} 
    & = \frac{B_j}{\|\Gamma(\mathcal{I}_j)-\Gamma(\mathcal{I}^*)\|} \\
    & = 2 \frac{\sum_{i\in\Omega\backslash\mathcal{I}^*} x_i - \sum_{i\in\Omega\backslash\mathcal{I}_j} x_i}{2\sqrt{|\mathcal{I}^*-\mathcal{I}_j|}/n} \\
    & =  n \frac{\sum_{i\in( \mathcal{I}_j-\mathcal{I}^*)} x_i - \sum_{i\in(\mathcal{I}^*-\mathcal{I}_j)} x_i}{\sqrt{|\mathcal{I}^*-\mathcal{I}_j|}},
\end{align*}
where the second line can be obtained by substituting the definition of $B_j$. Notice that $\mathcal{I}_j-\mathcal{I}^*\in\Omega\backslash \mathcal{I}^*$ and $\mathcal{I}^*-\mathcal{I}_j \in \mathcal{I}^*$. Any element with index in $\Omega \backslash \mathcal{I}^*$ is larger than any element in $\mathcal{I}^*$ by at least $x_{\sigma_{K+1}}-x_{\sigma_{K}}$. Then we have
\begin{align*}
    \frac{B_j}{\|e_j-e_0\|} 
    & =  N \frac{\sum_{i\in( \mathcal{I}_j-\mathcal{I}^*)} x_i - \sum_{i\in(\mathcal{I}^*-\mathcal{I}_j)} x_i}{\sqrt{|\mathcal{I}^*-\mathcal{I}_j|}} \\
    & \geq N \frac{|\mathcal{I}^*-\mathcal{I}_j|(x_{\sigma_{K+1}}-x_{\sigma_{K}})}{\sqrt{|\mathcal{I}^*-\mathcal{I}_j|}} \\
    & \geq N(x_{\sigma_{K+1}}-x_{\sigma_{K}}),
\end{align*}
where the last step is because for $j\neq 0$, $|\mathcal{I}^*-\mathcal{I}_j|$ is at least $1$. 

Also notice that the value $n(x_{\sigma_{k+1}}-x_{\sigma_{k}})$ can be attained at $\mathcal{I}_{j^*} = \{\sigma_1, \cdots, \sigma_{k-1}, \sigma_{k+1}\}$. Therefore we have 
\begin{align*}
    \min_j \frac{B_j}{\|e_j-e_0\|} =n(x_{\sigma_{k+1}} -x_{\sigma_{k}}).
\end{align*}
\end{proof}

\noindent{\bf Theorem 2.}
Given a distinct sequence $\cX$ and its sorting permutation $\sigma$, with Euclidean square cost function, for the proposed top-$k$ solver we have 
\begin{align*}
    \|\Gamma^{*, \epsilon} - \Gamma^{*}\| \leq \frac{\epsilon (\ln n + \ln 2)}{n(x_{\sigma_{k+1}} -x_{\sigma_{k}}) }.
\end{align*}

\begin{proof}
This is a direct conclusion with Lemma \ref{thm:lm1} and Lemma \ref{thm:lm3}.
\end{proof}

\section{The Expression of the Gradient of $A^\epsilon$}
\label{sec:dev_gradient}
In this section we will derive the expression of $d A^{\epsilon}/d x_i$. We first list a few reminders that will be used later:
\begin{itemize}
    \item $\{x_i\}_{i=1}^n$ is a scalar set to be solved for top-$k$. $\{y_j\}_{j=1}^m$ is taken to be $\{0, 1\}$.
    \item $C\in \mathbb{R}^{n\times m}$ is the cost matrix, usually defined as $C_{ij} = (x_i-y_j)^2$. 
    \item The loss function of entropic optimal transport is 
    \begin{align*}
        \Gamma^{*, \epsilon} = \argmin_{\Gamma\in \Delta} f^{\epsilon} (\Gamma) = \argmin_{\Gamma\in \Delta} \langle C, \Gamma \rangle + \epsilon \sum_{i,j} \Gamma_{ij} \ln \Gamma_{ij},
    \end{align*}
    where $\Delta=\{\Gamma: \Gamma\in [0,1]^{n\times m}, \Gamma \bm{1}_m = \mu, \Gamma \bm{1}_n = \nu\}$.
    \item The dual problem of the above optimization problem is
    \begin{align*}
        \xi^*, \zeta^* = \argmax_{\xi, \zeta} \mathcal{L}(\xi, \zeta; C),
    \end{align*}
    where
    \begin{align*}
        \mathcal{L}(\xi, \zeta; C) = \xi^\top \mu + \zeta^\top \nu -\epsilon \sum_{i,j=1}^{n,m} e^{-\frac{C_{ij}-\xi_i-\zeta_j}{\epsilon}}.
    \end{align*}
    And it is connected to the prime form by 
    \begin{align*}
        \Gamma^{*, \epsilon} = {\rm diag}(e^{\frac{\xi^*}{\epsilon}})e^{-\frac{C}{\epsilon}}{\rm diag}(e^{\frac{\zeta^*}{\epsilon}}).
    \end{align*}
    The converged $p,q$ in Algorithm \ref{alg:topk} is actually $e^{\frac{\xi^*}{\epsilon}}$ and $e^{\frac{\zeta^*}{\epsilon}}$.
\end{itemize}

If we obtain the expression for $\frac{d \xi^*}{d C}$ and $\frac{d \zeta^*}{d C}$, we can obtain the expression for  $\frac{d A^{\epsilon}}{d x_i}$. 

In this section only, we denote $\Gamma = \Gamma^{*, \epsilon}$, to shorten the notation. The multiplication of $3$rd-order tensors mirrors the multiplication of matrices: we always use the last dimension of the first input to multiplies the first dimension of the second input. We denote $\bar{b}=b_{:-1}$ as $b$ removing the last entry, $\bar{\nu}=\nu_{:-1}$ as $\nu$ removing the last entry, $\bar{\Gamma} = \Gamma_{:,:-1}$ as $\Gamma$ removing the last column.

\begin{theorem} $\frac{d \xi^*}{d C}$ and $\frac{d \zeta^*}{d C}$ have the following expression,
\begin{align*}
\begin{bmatrix}
    \frac{d \xi^*}{d C}\\
    \frac{d \zeta^*}{d C}
\end{bmatrix} = 
\begin{bmatrix}
 - H^{-1} D \\
 \bm{0}
\end{bmatrix}
\end{align*}
where $- H^{-1} D\in \mathbb{R}^{(n+m-1)\times n \times m}$, $\bm{0}\in \mathbb{R}^{1\times n \times m}$, and
\begin{align*}
    & D_{\ell ij} = \frac{1}{\epsilon} \begin{cases}
    \delta_{\ell i} \Gamma_{ij}, \ell=1, \cdots, n \\
    \delta_{\ell j} \Gamma_{ij}, \ell=n+1, \cdots, n+m-1 
    \end{cases} \\
    & H^{-1} = -{\epsilon} \begin{bmatrix} 
    ({\rm diag}(\mu))^{-1} + ({\rm diag}(\mu))^{-1} \bar{\Gamma} \mathcal{K}^{-1} \bar{\Gamma}^T ({\rm diag}(\mu))^{-1} & - ({\rm diag}(\mu))^{-1} \bar{\Gamma} \mathcal{K}^{-1} \\
    -\mathcal{K}^{-1} \bar{\Gamma} ^T ({\rm diag}(\mu))^{-1} & \mathcal{K}^{-1}
    \end{bmatrix} \\
    & \mathcal{K} = {\rm diag}(\bar{\nu}) - \bar{\Gamma}^T ({\rm diag}(\mu))^{-1} \bar{\Gamma}.
\end{align*}
\end{theorem}
\begin{proof} Notice that there is one redundant dual variable, since $\mu \bm{1}_N = \nu \bm{1}_M=1$. Therefore, we can rewrite $\mathcal{L}(\xi, \zeta; C)$ as
\begin{align*}
        \mathcal{L}(\xi,\bar{\zeta}; C) = \xi^T \mu + \bar{\zeta}^T \bar{\nu} -\epsilon \sum_{i,j=1}^{n,m-1} e^{\frac{-C_{ij}+\xi_i+\zeta_j}{\epsilon}}-\epsilon \sum_{i=1}^{n} e^{\frac{-C_{im}+\xi_i}{\epsilon}}.
\end{align*}
Denote
\begin{align}
    &\phi(\xi, \bar{\zeta},C) = \frac{d \mathcal{L}(\xi, \bar{\zeta};C)}{d \xi}
     = \mu - F \bm{1}_m, \label{eq:phi} \\
   & \psi(\xi, \bar{\zeta},C) = \frac{d \mathcal{L}(\xi, \bar{\zeta};C)}{d \bar{\zeta}}
     = \bar{\nu} - \bar{F}^\top \bm{1}_n, \label{eq:psi} 
\end{align}
where 
\begin{align}
     & F_{ij}  = e^{\frac{-C_{ij}+\xi_i+\zeta_j}{\epsilon}}, \quad\forall i=1, \cdots, n, \quad j=1, \cdots, m-1 \nonumber \\
     & F_{im} = e^{\frac{-C_{im}+\xi_i}{\epsilon}}, \quad\forall i=1, \cdots, n, \nonumber\\
     & \bar{F}  = F_{:,:-1}. \nonumber
\end{align}
Since $(\xi^*, \bar{\zeta}^*)$ is a maximum of $\mathcal{L}(\xi,\bar{\zeta}; C)$, we have
\begin{align*}
    & \phi(\xi^*, \bar{\zeta}^*,C)=0, \\
    & \psi(\xi^*, \bar{\zeta}^*,C)=0 .
\end{align*}
Therefore,
\begin{align*}
    & \frac{d \phi(\xi^*, \bar{\zeta}^*,C)}{d C} = \frac{\partial \phi(\xi^*, \bar{\zeta}^*,C)}{\partial C} + \frac{\partial \phi(\xi^*, \bar{\zeta}^*,C)}{\partial \xi^*} \frac{d \xi^*}{d C}+ \frac{\partial \phi(\xi^*, \bar{\zeta}^*,\mu,\nu,C)}{\partial \bar{\zeta}^*} \frac{d \bar{\zeta}^*}{d C} = 0, \\
    & \frac{d \psi(\xi^*, \bar{\zeta}^*,C)}{d C} = \frac{\partial \psi(\xi^*, \bar{\zeta}^*,C)}{\partial C} + \frac{\partial \psi(\xi^*, \bar{\zeta}^*,C)}{\partial \xi^*} \frac{d \xi^*}{d C}+ \frac{\partial \psi(\xi^*, \bar{\zeta}^*,C)}{\partial \bar{\zeta}^*} \frac{d \bar{\zeta}^*}{d C} = 0.
\end{align*}
Therefore,
\begin{align*}
    \begin{bmatrix}
    \frac{d \xi^*}{d C} \\
    \frac{d \bar{\zeta}^*}{d C}
    \end{bmatrix} & = -
    \begin{bmatrix}
    \frac{\partial \phi(\xi^*, \bar{\zeta}^*,C)}{\partial \xi^*} & \frac{\partial \phi(\xi^*, \bar{\zeta}^*,C)}{\partial \bar{\zeta}^*}  \\
    \frac{\partial \psi(\xi^*, \bar{\zeta}^*,C)}{\partial \xi^*} & \frac{\partial \psi(\xi^*, \bar{\zeta}^*,C)}{\partial \bar{\zeta}^*}
    \end{bmatrix}^{-1}  
    \begin{bmatrix}
    \frac{\partial \phi(\xi^*, \bar{\zeta}^*,C)}{\partial C} \\
    \frac{\partial \psi(\xi^*, \bar{\zeta}^*,C)}{\partial C}
    \end{bmatrix} \\
    & \triangleq -
    H^{-1}  
    \begin{bmatrix}
    D^{(1)} \\
    D^{(2)}
    \end{bmatrix} \\
    & \triangleq - H^{-1} D.
\end{align*}
Now let's compute each of the terms.
\begin{align*}
    \frac{\partial \phi(\xi^*,\bar{\zeta}^*,C)_h}{\partial C_{ij}} & = - \frac{\partial [F\bm{1}_m]_h}{\partial C_{ij}}
    = - \frac{\partial}{\partial C_{ij}} \left(\sum_{\ell=1}^{m-1} e^{\frac{-C_{h \ell}+a_h+b_\ell}{\epsilon}} + e^{\frac{-C_{hm}+a_h}{\epsilon}} \right) \\
    & = \frac{1}{\epsilon} \delta_{hi} F_{ij} =  \frac{1}{\epsilon} \delta_{hi} \Gamma_{ij}\\
    &\forall h=1, \cdots, n, \quad i=1, \cdots, n, \quad j=1, \cdots, m \\
    \frac{\partial \psi(\xi^*,\bar{\zeta}^*,C)_\ell}{\partial C_{ij}} & = - \frac{\partial [\bar{F}^\top \bm{1}_n]_\ell}{\partial C_{ij}}
    = - \frac{\partial}{\partial C_{ij}} \sum_{h=1}^{n} e^{\frac{-C_{h \ell}+a_h+b_\ell}{\epsilon}}  \\
    & = \frac{1}{\epsilon} \delta_{\ell j} F_{ij} =  \frac{1}{\epsilon} \delta_{\ell j} \Gamma_{ij}\\
    &\forall \ell=1, \cdots, m-1, \quad i=1, \cdots, n, \quad j=1, \cdots, m \\
    \frac{\partial \phi(\xi^*, \bar{\zeta}^*,C)_h}{\partial \xi^*_i} & = - \frac{\partial [F\bm{1}_m]_h}{\partial \xi^*_{i}}
    = - \frac{\partial}{\partial \xi^*_{i}} \left(\sum_{\ell=1}^{m-1} e^{\frac{-C_{h\ell}+a_h+b_\ell}{\epsilon}} + e^{\frac{-C_{hm}+a_h}{\epsilon}} \right) \\
    & = - \frac{1}{\epsilon} \delta_{hi} \sum_{\ell=1}^{m} F_{h \ell} =  - \frac{1}{\epsilon} \delta_{hi} \mu_{h} \\
    & \forall h=1, \cdots, n, \quad i=1, \cdots, n \\
    \frac{\partial \phi(\xi^*, \bar{\zeta}^*,C)_h}{\partial \bar{\zeta}^*_j} & = - \frac{\partial [F\bm{1}_m]_h}{\partial \bar{\zeta}^*_{j}}
    = - \frac{\partial}{\partial \bar{\zeta}^*_{j}} \left(\sum_{\ell=1}^{m-1} e^{\frac{-C_{h\ell}+a_h+b_\ell}{\epsilon}} + e^{\frac{-C_{hm}+a_h}{\epsilon}} \right) \\
    & = - \frac{1}{\epsilon} \sum_{\ell=1}^{m-1}\delta_{\ell j}  F_{h\ell} =- \frac{1}{\epsilon}   F_{hj} =  - \frac{1}{\epsilon} \Gamma_{hj} \\
    & \forall h=1, \cdots, n, \quad j=1, \cdots, m-1 \\
    \frac{\partial \psi(\xi^*, \bar{\zeta}^*,C)_\ell}{\partial \xi^*_i} & = - \frac{\partial [\bar{F}^\top\bm{1}_n]_\ell}{\partial \xi^*_{i}}
    = - \frac{\partial}{\partial \xi^*_{i}}\sum_{h=1}^{n} e^{\frac{-C_{h\ell}+a_h+b_\ell}{\epsilon}}  \\
    & = - \frac{1}{\epsilon} \sum_{h=1}^{n}\delta_{hi}  F_{h\ell} =- \frac{1}{\epsilon}   F_{i\ell} =  - \frac{1}{\epsilon} \Gamma_{i\ell} \\
    & \forall \ell=1, \cdots, m-1, \quad i=1, \cdots, n \\
    \frac{\partial \psi(\xi^*, \bar{\zeta}^*,C)_\ell}{\partial \bar{\zeta}^*_{j}} & = - \frac{\partial [\bar{F}^\top\bm{1}_n]_\ell}{\partial \bar{\zeta}^*_{j}}
    = - \frac{\partial}{\partial \bar{\zeta}^*_{j}}\sum_{h=1}^{n} e^{\frac{-C_{h\ell}+a_h+b_\ell}{\epsilon}}  \\
    & = - \frac{1}{\epsilon} \sum_{h=1}^{n}\delta_{\ell j}  F_{h\ell} = - \frac{1}{\epsilon} \delta_{\ell j} \nu_{\ell} \\
    & \forall \ell=1, \cdots, m-1, \quad j=1, \cdots, m-1.
\end{align*}
To sum up, we have
\begin{align*}
    H = -\frac{1}{\epsilon}\begin{bmatrix}
        {\rm diag}(\mu) & \bar{\Gamma} \\
        \bar{\Gamma}^T & {\rm diag}(\bar{\nu})
    \end{bmatrix}.
\end{align*}
Following the formula for inverse of block matrices,
\begin{align*}
    \begin{bmatrix}
    \mathbf{A} & \mathbf{B} \\
    \mathbf{C} & \mathbf{D}
  \end{bmatrix}^{-1} = \begin{bmatrix}
     \mathbf{A}^{-1} + \mathbf{A}^{-1}\mathbf{B}(\mathbf{D} - \mathbf{CA}^{-1}\mathbf{B})^{-1}\mathbf{CA}^{-1} &
      -\mathbf{A}^{-1}\mathbf{B}(\mathbf{D} - \mathbf{CA}^{-1}\mathbf{B})^{-1} \\
    -(\mathbf{D}-\mathbf{CA}^{-1}\mathbf{B})^{-1}\mathbf{CA}^{-1} &
       (\mathbf{D} - \mathbf{CA}^{-1}\mathbf{B})^{-1}
  \end{bmatrix},
\end{align*}
denote 
\begin{align*}
    \mathcal{K} = {\rm diag}(\bar{\nu}) - \bar{\Gamma}^T ({\rm diag}(\mu))^{-1} \bar{\Gamma}.
\end{align*}
Note that $\cK$ is just a scalar for SOFT top-$k$ operator, and is a $(k-1)\times (k-1)$ matrix for sorted SOFT top-$k$ operator. Therefore computing its inverse is not expensive. Finally we have
\begin{align*}
    H^{-1} = -\epsilon \begin{bmatrix} 
    ({\rm diag}(\mu))^{-1} + ({\rm diag}(\mu))^{-1} \bar{\Gamma} \mathcal{K}^{-1} \bar{\Gamma}^T ({\rm diag}(\mu))^{-1} & - ({\rm diag}(\mu))^{-1} \bar{\Gamma} \mathcal{K}^{-1} \\
    -\mathcal{K}^{-1} \bar{\Gamma} ^T ({\rm diag}(\mu))^{-1} & \mathcal{K}^{-1}
    \end{bmatrix}.
\end{align*}
And also
\begin{align*}
    & D^{(1)}_{hij} = \frac{1}{\epsilon} \delta_{hi}\Gamma_{ij} \\
    & D^{(2)}_{\ell ij} = \frac{1}{\epsilon} \delta_{\ell j}\Gamma_{ij}.
\end{align*}
The above derivation can actually be viewed as we explicitly force $b_m=0$, i.e., no matter how $C$ changes, $b_m$ does not change. Therefore, we can treat $\frac{db_m}{dC}=\bm{0}_{n\times m}$, and we get the equation in the theorem. 
\end{proof}

After we obtain $\frac{d\xi^*}{dC}$ and $\frac{d\zeta^*}{dC}$, we can now compute $\frac{d\Gamma}{dC}$.
\begin{align*}
    \frac{d\Gamma_{h\ell}}{d C_{ij}} = \frac{d}{dC_{ij}} e^{\frac{-C_{h\ell}+a_h+b_\ell}{\epsilon}} = \frac{1}{\epsilon} \left( -\Gamma_{h\ell}\delta_{ih}\delta_{j\ell} + \Gamma_{h\ell}\frac{d\xi^*_h}{dC_{ij}} + \Gamma_{h\ell}\frac{db^*_\ell}{dC_{ij}} \right).
\end{align*}
Finally, in the back-propagation step, we can compute the gradient of the loss $L$ w.r.t. $C$,
\begin{align*}
    \frac{dL}{d C_{ij}} & = \sum_{h,\ell=1}^{n,m} \frac{dL}{d\Gamma_{h\ell}} \frac{d\Gamma_{h\ell}}{d C_{ij}} \\
    & = \frac{1}{\epsilon} \left( -\sum_{h,\ell=1}^{n,m} \frac{dL}{d\Gamma_{h\ell}}\Gamma_{h\ell}\delta_{in}\delta_{j\ell} + \sum_{h,\ell=1}^{n,m} \frac{dL}{d\Gamma_{h\ell}}\Gamma_{h\ell}\frac{d\xi^*_h}{dC_{ij}} + \sum_{h,\ell=1}^{n,m} \frac{dL}{d\Gamma_{h\ell}}\Gamma_{h\ell}\frac{db^*_\ell}{dC_{ij}} \right) \\
    & = \frac{1}{\epsilon} \left( - \frac{dL}{d\Gamma_{ij}}\Gamma_{ij} + \sum_{h,\ell=1}^{n,m} \frac{dL}{d\Gamma_{h\ell}}\Gamma_{h\ell}\frac{d\xi^*_h}{dC_{ij}} + \sum_{h,\ell=1}^{n,m} \frac{dL}{d\Gamma_{h\ell}}\Gamma_{h\ell}\frac{db^*_\ell}{dC_{ij}} \right).
\end{align*}

We summarize the above procedure for computing the gradient for sorted SOFT top-$k$ operator in Algorithm \ref{alg:topk_grad}. This naive implementation takes $\cO(n^2 k)$ complexity, which is not efficient. Therefore, we modify the algorithm using the associative law of matrix multiplications, so that the complexity is lowered to $\cO(nk)$. We summarize the modified algorithm in Algorithm \ref{alg:topk_grad2}. 

We also include the \texttt{PyTorch} implementation of the forward pass and backward pass as shown below. The code is executed by creating an instance of \texttt{TopK\_custom}, and the forward pass and the backward pass is run similar to any other \texttt{PyTorch} model.

\begin{algorithm}
\caption{\label{alg:topk_grad} Gradient for Sorted Top-$K$}
\begin{algorithmic} 
\REQUIRE $C \in \mathbb{R}^{n\times(k+1)}, \mu\in \mathbb{R}^{n}, \nu\in \mathbb{R}^{k+1}, \frac{d \mathcal{L}}{d\Gamma}\in \mathbb{R}^{n\times(k+1)}, \epsilon$
\STATE Run forward pass to get $\Gamma$
\STATE $\bar{\nu} = \nu[:-1], \bar{\Gamma} = \Gamma[:,:-1]$
\STATE $\mathcal{K} \leftarrow {\rm diag}(\bar{\nu}) - \bar{\Gamma}^T ({\rm diag}(\mu))^{-1} \bar{\Gamma}$\myCOMMENT{$\mathcal{K}\in \mathbb{R}^{k\times k}$}
\STATE $H1 \leftarrow ({\rm diag}(\mu))^{-1} + ({\rm diag}(\mu))^{-1} \bar{\Gamma} \mathcal{K}^{-1} \bar{\Gamma}^T ({\rm diag}(\mu))^{-1} $ \myCOMMENT{$H1\in \mathbb{R}^{n\times n}$}
\STATE $H2 \leftarrow - ({\rm diag}(\mu))^{-1} \bar{\Gamma} \mathcal{K}^{-1}$ \myCOMMENT{$H2\in \mathbb{R}^{n\times k}$}
\STATE $H3 \leftarrow  (H2)^T $ \myCOMMENT{$H3\in \mathbb{R}^{k\times n}$}
\STATE $H4 \leftarrow \mathcal{K}^{-1} $ \myCOMMENT{$H4\in \mathbb{R}^{k\times k}$}
\STATE Pad $H2$ to be $[n, k+1]$ in the last column with value $0$
\STATE Pad $H4$ to be $[k, k+1]$ in the last column with value $0$
\STATE $[\frac{d \xi^*}{dC}]_{hij} \leftarrow  [H1]_{hi}\Gamma_{ij}+[H2]_{hj}\Gamma_{ij}$ \myCOMMENT{$\frac{d \xi^*}{dC} \in \mathbb{R}^{n\times n\times (k+1)}$}
\STATE $[\frac{d b^*}{dC}]_{\ell ij} \leftarrow  [H3]_{\ell i}\Gamma_{ij}+[H4]_{\ell j}\Gamma_{ij}$ \myCOMMENT{$\frac{d b^*}{dC} \in \mathbb{R}^{k\times n\times (k+1)}$}
\STATE Pad $\frac{d b^*}{dC}$ to be $[k+1, n, k+1]$ with value $0$
\STATE $[\frac{d \mathcal{L}}{dC}]_{ij} \leftarrow \frac{1}{\epsilon}(-[\frac{d \mathcal{L}}{d\Gamma}]_{ij}\Gamma_{ij} + \sum_{h,\ell}[\frac{d \mathcal{L}}{d\Gamma}]_{h\ell}\Gamma_{h\ell}[\frac{d \xi^*}{dC}]_{hij} + \sum_{h,\ell}[\frac{d \mathcal{L}}{d\Gamma}]_{h\ell}\Gamma_{h\ell}[\frac{d b^*}{dC}]_{\ell ij} )$
\end{algorithmic}
\end{algorithm}

\begin{algorithm}
\caption{\label{alg:topk_grad2} Gradient for Sorted Top-$k$, with reduced memory}
\begin{algorithmic} 
\REQUIRE $C \in \mathbb{R}^{N\times(K+1)}, \mu\in \mathbb{R}^{N}, \nu\in \mathbb{R}^{K+1}, \frac{d \mathcal{L}}{d\Gamma}\in \mathbb{R}^{N\times(K+1)}, \epsilon$
\STATE Run forward pass to get $\Gamma$
\STATE $\bar{\nu} = \nu[:-1], \bar{\Gamma} = \Gamma[:,:-1]$
\STATE $\mathcal{K} \leftarrow {\rm diag}(\bar{\nu}) - \bar{\Gamma}^T ({\rm diag}(\mu))^{-1} \bar{\Gamma}$
\myCOMMENT{$\mathcal{K}\in \mathbb{R}^{K\times K}$}
\STATE $\mu'_i=\mu^{-1}_i$
\STATE $L \leftarrow ({\rm diag}(\mu))^{-1} \bar{\Gamma} \mathcal{K}^{-1}$ 
\myCOMMENT{$L\in \mathbb{R}^{N\times K}$}
\STATE $G1 \leftarrow \frac{d \mathcal{L}}{d\Gamma}\odot\Gamma$ 
\myCOMMENT{$G1 \in \mathbb{R}^{N\times K}$}
\STATE $g1 \leftarrow [G1]\bm{1}_K$, $g2 \leftarrow [G1]^T\bm{1}_N$ 
\myCOMMENT{$g1\in \mathbb{R}^{N}, g2\in \mathbb{R}^{K}$}
\STATE $G21 \leftarrow  (g1\odot \mu').expand\_dims(1)\odot \Gamma$ 
\myCOMMENT{$G21 \in \mathbb{R}^{N\times (K+1)}$}
\STATE $G22 \leftarrow ((g1)^T L \bar{\Gamma}^T \odot \mu').expand\_dims(1)\odot \Gamma $
\myCOMMENT{$G22 \in \mathbb{R}^{N\times (K+1)}$}
\STATE $G23 \leftarrow -((g1)^T L).pad\_last\_entry(0).expand\_dims(0)\odot \Gamma  $
\myCOMMENT{$G23 \in \mathbb{R}^{N\times (K+1)}$}
\STATE $G2 = G21 + G22 + G23$
\myCOMMENT{$G2 \in \mathbb{R}^{N\times (K+1)}$}
\STATE $g2 \leftarrow g2[:-1]$
\STATE $G31 \leftarrow -(L (g2)).expand\_dims(1)\odot \Gamma  $
\myCOMMENT{$G31 \in \mathbb{R}^{N\times (K+1)}$}
\STATE $G32 \leftarrow (\mathcal{K}^{-1} (g2)).pad\_last\_entry(0).expand\_dims(0)\odot \Gamma  $
\myCOMMENT{$G32 \in \mathbb{R}^{N\times (K+1)}$}
\STATE $G3 = G31 + G32$
\myCOMMENT{$G3 \in \mathbb{R}^{N\times (K+1)}$}
\STATE $\frac{d \mathcal{L}}{dC} \leftarrow \frac{1}{\epsilon}(-G1 + G2 + G3 )$
\end{algorithmic}
\end{algorithm}

\newpage
\begin{lstlisting}[
    basicstyle=\tiny, %or \small or \footnotesize etc.
]
def sinkhorn_forward(C, mu, nu, epsilon, max_iter):
    bs, n, k_ = C.size()

    v = torch.ones([bs, 1, k_])/(k_)
    G = torch.exp(-C/epsilon)
    if torch.cuda.is_available():
        v = v.cuda()

    for i in range(max_iter):
        u = mu/(G*v).sum(-1, keepdim=True)
        v = nu/(G*u).sum(-2, keepdim=True)

    Gamma = u*G*v
    return Gamma

def sinkhorn_forward_stablized(C, mu, nu, epsilon, max_iter):
    bs, n, k_ = C.size()
    k = k_-1

    f = torch.zeros([bs, n, 1])
    g = torch.zeros([bs, 1, k+1])
    if torch.cuda.is_available():
        f = f.cuda()
        g = g.cuda()

    epsilon_log_mu = epsilon*torch.log(mu)
    epsilon_log_nu = epsilon*torch.log(nu)

    def min_epsilon_row(Z, epsilon):
        return -epsilon*torch.logsumexp((-Z)/epsilon, -1, keepdim=True)
    
    def min_epsilon_col(Z, epsilon):
        return -epsilon*torch.logsumexp((-Z)/epsilon, -2, keepdim=True)

    for i in range(max_iter):
        f = min_epsilon_row(C-g, epsilon)+epsilon_log_mu
        g = min_epsilon_col(C-f, epsilon)+epsilon_log_nu
        
    Gamma = torch.exp((-C+f+g)/epsilon)
    return Gamma
    
def sinkhorn_backward(grad_output_Gamma, Gamma, mu, nu, epsilon):
    
    nu_ = nu[:,:,:-1]
    Gamma_ = Gamma[:,:,:-1]

    bs, n, k_ = Gamma.size()
    
    inv_mu = 1./(mu.view([1,-1]))  #[1, n]
    Kappa = torch.diag_embed(nu_.squeeze(-2)) \
            -torch.matmul(Gamma_.transpose(-1, -2) * inv_mu.unsqueeze(-2), Gamma_)   #[bs, k, k]
    
    inv_Kappa = torch.inverse(Kappa) #[bs, k, k]
    
    Gamma_mu = inv_mu.unsqueeze(-1)*Gamma_
    L = Gamma_mu.matmul(inv_Kappa) #[bs, n, k]
    G1 = grad_output_Gamma * Gamma #[bs, n, k+1]
    
    g1 = G1.sum(-1)
    G21 = (g1*inv_mu).unsqueeze(-1)*Gamma  #[bs, n, k+1]
    g1_L = g1.unsqueeze(-2).matmul(L)  #[bs, 1, k]
    G22 = g1_L.matmul(Gamma_mu.transpose(-1,-2)).transpose(-1,-2)*Gamma  #[bs, n, k+1]
    G23 = - F.pad(g1_L, pad=(0, 1), mode='constant', value=0)*Gamma  #[bs, n, k+1]
    G2 = G21 + G22 + G23  #[bs, n, k+1]
    
    del g1, G21, G22, G23, Gamma_mu
    
    g2 = G1.sum(-2).unsqueeze(-1) #[bs, k+1, 1]
    g2 = g2[:,:-1,:]  #[bs, k, 1]
    G31 = - L.matmul(g2)*Gamma  #[bs, n, k+1]
    G32 = F.pad(inv_Kappa.matmul(g2).transpose(-1,-2), pad=(0, 1), mode='constant', value=0)*Gamma  #[bs, n, k+1]
    G3 = G31 + G32  #[bs, n, k+1]

    grad_C = (-G1+G2+G3)/epsilon  #[bs, n, k+1]
    return grad_C

class TopKFunc(Function):
    @staticmethod
    def forward(ctx, C, mu, nu, epsilon, max_iter):
        
        with torch.no_grad():
            if epsilon>1e-2:
                Gamma = sinkhorn_forward(C, mu, nu, epsilon, max_iter)
                if bool(torch.any(Gamma!=Gamma)):
                    print('Nan appeared in Gamma, re-computing...')
                    Gamma = sinkhorn_forward_stablized(C, mu, nu, epsilon, max_iter)
            else:
                Gamma = sinkhorn_forward_stablized(C, mu, nu, epsilon, max_iter)
            ctx.save_for_backward(mu, nu, Gamma)
            ctx.epsilon = epsilon
        return Gamma

    @staticmethod
    def backward(ctx, grad_output_Gamma):
        
        epsilon = ctx.epsilon
        mu, nu, Gamma = ctx.saved_tensors
        # mu [1, n, 1]
        # nu [1, 1, k+1]
        #Gamma [bs, n, k+1]   
        with torch.no_grad():
            grad_C = sinkhorn_backward(grad_output_Gamma, Gamma, mu, nu, epsilon)
        return grad_C, None, None, None, None


class TopK_custom(torch.nn.Module):
    def __init__(self, k, epsilon=0.1, max_iter = 200):
        super(TopK_custom1, self).__init__()
        self.k = k
        self.epsilon = epsilon
        self.anchors = torch.FloatTensor([k-i for i in range(k+1)]).view([1,1, k+1])
        self.max_iter = max_iter
        
        if torch.cuda.is_available():
            self.anchors = self.anchors.cuda()

    def forward(self, scores):
        bs, n = scores.size()
        scores = scores.view([bs, n, 1])
        
        #find the -inf value and replace it with the minimum value except -inf
        scores_ = scores.clone().detach()
        max_scores = torch.max(scores_).detach()
        scores_[scores_==float('-inf')] = float('inf')
        min_scores = torch.min(scores_).detach()
        filled_value = min_scores - (max_scores-min_scores)
        mask = scores==float('-inf')
        scores = scores.masked_fill(mask, filled_value)
        
        C = (scores-self.anchors)**2
        C = C / (C.max().detach())
      
        mu = torch.ones([1, n, 1], requires_grad=False)/n
        nu = [1./n for _ in range(self.k)]
        nu.append((n-self.k)/n)
        nu = torch.FloatTensor(nu).view([1, 1, self.k+1])
        
        if torch.cuda.is_available():
            mu = mu.cuda()
            nu = nu.cuda()
            
        Gamma = TopKFunc.apply(C, mu, nu, self.epsilon, self.max_iter)
 
        A = Gamma[:,:,:self.k]*n
        
        return A, None
\end{lstlisting}

\section{Experiment Settings}
\label{sec:exp_setting}

\subsection{$k$NN}

The settings of the neural networks, the training procedure, and the number of neighbors $k$, and the tuning procedures are similar to \citet{grover2019stochastic}. The tuning o $\epsilon$ ranging from $10^{-6}$ to $10^{-2}$. Other settings are shown in Table \ref{tab:knn_setting}.

\begin{table}[h]
\centering
\caption{\label{tab:knn_setting} Parameter settings for $k$NN experiments.}
\begin{tabular}{lll}
\hline
Dataset      & MNIST  & CIFAR-10\\ \hline
$k$   & $9$ & $9$ \\
$\epsilon$    & $10^{-3}$ & $10^{-5}$ \\
Batch size of query samples    & $100$ & $100$ \\
Batch size of template samples    & $100$ & $100$ \\
Optimizer   & SGD & SGD \\
Learning rate   & $10^{-3}$ & $10^{-3}$ \\
Momentum    & $0.9$ & $0.9$ \\
Weight decay  & $5\times 10^{-4}$ & $5\times 10^{-4}$\\
Model & 2-layer convolutional network & ResNet18 \\
\hline
\end{tabular}
\vspace{-0.1in}
\end{table}
Note that $f_\theta$ is a feature extraction neural network, so that model specified in the last row of Table \ref{tab:knn_setting} does not contain the final activation layer and the linear layer.

\noindent{\bf Baselines.} In the baselines, the results of $k$NN, $k$NN+PCA, $k$NN+AE, $k$NN+NeuralSort is copied from \citet{grover2019stochastic}. The result of RelaxSubSample is copied from \citet{xie2019reparameterizable}.

The implementation of $k$NN+\citet{cuturi2019differentiable} is based on \citet{grover2019stochastic}. Specifically, the outputs of the models in \citet{cuturi2019differentiable} and \citet{grover2019stochastic} are both doubly stochastic matrices. So in the implementation of $k$NN+\citet{cuturi2019differentiable}, we adopt the algorithm in \citet{grover2019stochastic}, except that we replace the module of computing the doubly stochastic matrix to be the one in \citet{cuturi2019differentiable}. We extensively tuned $k$, $\epsilon$ and the learning rate, but cannot achieve a better score for this experiment. 

The baselines $k$NN+Softmax $k$ times, $k$NN+pretrained CNN, and CE+CNN adopts the identical neural networks as our model. We remark that the scores reported in \citet{grover2019stochastic} for CNN+CE are $99.4\%$ for MNIST and $95.1\%$ for CIFAR-10. However, our experiments \textit{using their code} cannot reproduce the reported scores: and the scores are $99.0\%$ and $90.9\%$, respectively. Therefore, the reported score for MNIST is implemented by us, and the score for CIFAR-10 is copied from \citet{he2016deep}.

\subsection{Beam Search}

\noindent{\bf Algorithm.} We now elaborate how to backtrack the predecessors $E^{(1:t),r}$ for an embedding $E^{(t+1),\ell}$, and how to compute the likelihood $\cL_{\rm s}(E^{(1:t+1),\ell})$, which we have omitted in Algorithm \ref{alg:beam_search2}. Specifically, in standard beam search algorithm, each selected token $\tilde{y}^{(t+1),\ell}$ is generated from a specific predecessor, and thus the backtracking is straightforward. In beam search with sorted SOFT top-$k$ operator, however, each computed embedding $E^{(1:t),r}$ is a weighted sum of the output from all predecessors, so that it is not corresponding to one specific predecessor. To address this difficulty, we select the predecessor for $E^{(t+1),\ell}$ with the largest weight, i.e.,
\begin{align*}
    (o,r) = \argmax_{(j,i)} A^{(t),\epsilon}_{ji,\ell}.
\end{align*}
This is a good approximation because $A^{(t), \epsilon}$ is a smoothed $0$-$1$ tensor, i.e., for each $\ell$, there is only one entry that is approximately $1$ in $A^{(t), \epsilon}_{:,:,\ell}$, while the others are approximately $0$. The likelihood is then computed as follows
\begin{align*}
    \cL_{\rm s}(E^{(1:t+1),\ell}) = \cL_{\rm s} (E^{(1:t),r})\mathbb{P}(y^{t+1}=\omega_o |\tilde{h}^{(t),r}(E^{(1:t),r})).
\end{align*}

\noindent{\bf Implementation.} The implemented model is identical to \citet{bahdanau2014neural}. Different from \citet{bahdanau2014neural}, here we also preprocess the data with \textit{byte pair encoding} \citep{sennrich2015neural}.  

We adopt beam size $5$, teacher forcing ratio $\rho=0.8$, and $\epsilon=10^{-1}$.  The training procedure is as follows: We first pretrain the model with teacher forcing training procedure. The pretraining procedure has initial learning rate $1$, learning rate decay $0.1$ starting from iteration $5\times 10^{5}$ for every $10^{5}$ iterations. We pretrain it for $10^{6}$ iterations in total. We then train the model using the combined training procedure for $10^{5}$ iterations with learning rate $0.05$.

\subsection{Top-$k$ Attention}

The settings of the baseline model on data pre-processing, model, and the training procedure, evaluation procedure is identical to https://opennmt.net/OpenNMT-py/extended.html. The settings of the proposed model only differs in that we adopt SOFT top-$k$ attention instead of the standard soft attention.

\section{Visualization of the Gradients}

In this section we visualize the computed gradient using a toy example mimicking the settings of $k$NN classification. Specifically, we input $10$ scores computed from $10$ images, i.e., $\cX=\{0,1,2, \cdots,9\}$, into the SOFT top-$k$ operator, and select the top-$3$ elements. Denote the indices of the images with the same labels as the query sample as $\cI$. Similar to $k$NN classification, we want to maximize $\sum_{i\in\cI} A^{\epsilon}_i$.

We visualize the gradient on $\cX$ with respect to this objective function in Figure \ref{fig:grad}. In Figure \ref{fig:grad_a},  $\cI$ is the same as the indices of top-$3$ scores. In this case, the gradient will push the gap between the top-$3$ scores and the rest scores even further. In Figure \ref{fig:grad_b}, $\cI$ is different from the indices of top-$3$ scores. In this case, the scores corresponding to $\cI$ are pushed to be smaller, while the others are pushed to be larger.

\begin{figure}
\centering     
\subfigure[$\cI=\{0,1,2\}$.]{\label{fig:grad_a}\includegraphics[width=0.45\textwidth]{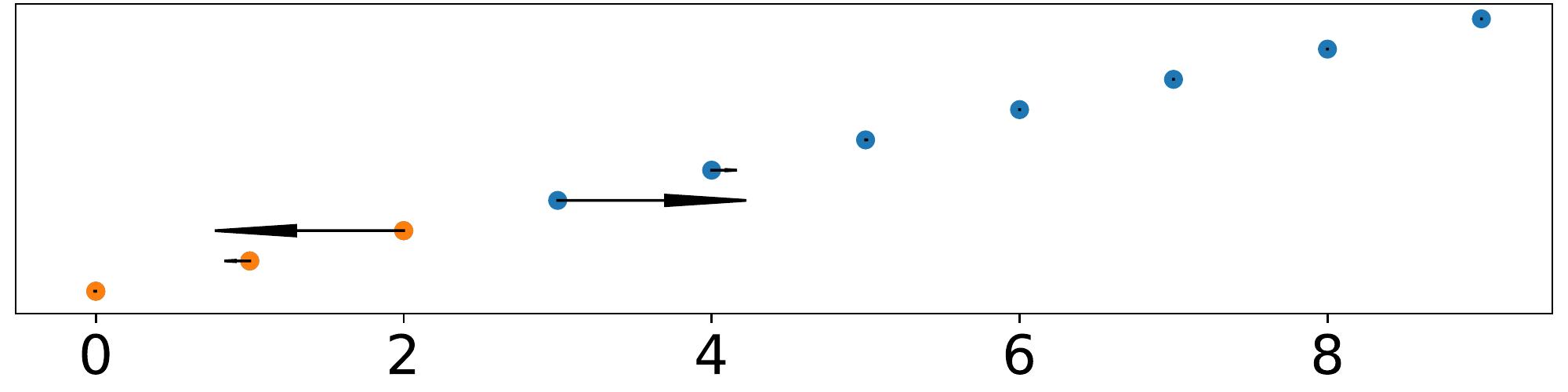}}
\subfigure[$\cI=\{2,3,4\}$.]{\label{fig:grad_b}\includegraphics[width=0.45\textwidth]{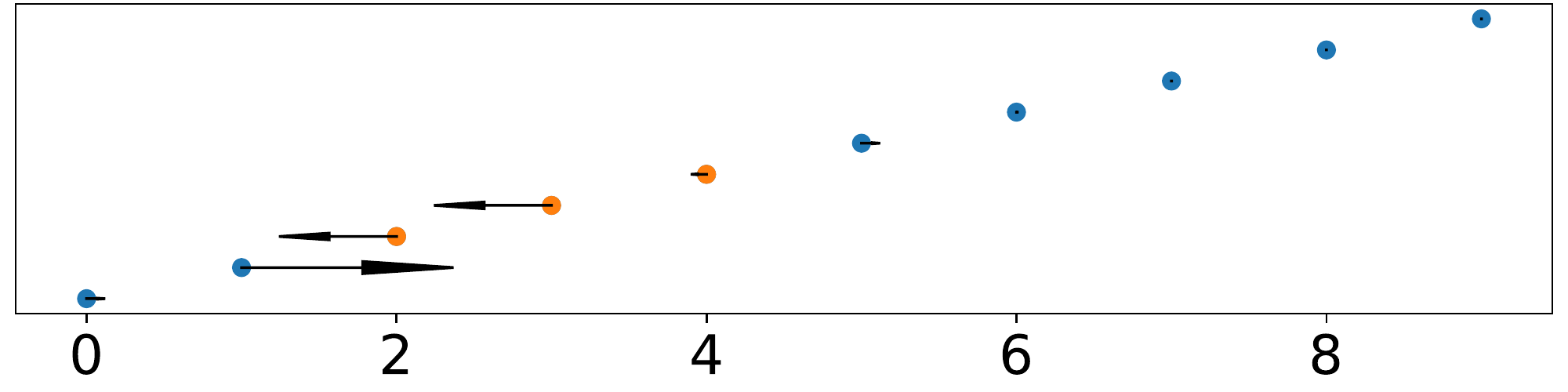}}
\vspace{-10pt}
\caption{Illustration of the gradient of the SOFT top-$k$ operators. The arrows represent the direction and magnitude of the gradient. The orange dots corresponds to the ground truth elements.}
\label{fig:grad}
\vspace{-20pt}
\end{figure}



\end{document}